\documentclass{article}

     \usepackage[preprint,nonatbib]{neurips_2019}

\usepackage[utf8]{inputenc} 
\usepackage[T1]{fontenc}    
\usepackage{url}            
\usepackage{booktabs}       
\usepackage{amsfonts}       
\usepackage{nicefrac}       
\usepackage{microtype}      
\usepackage{graphicx}
\usepackage{ascmac}
\usepackage{amsmath,amssymb,amsfonts,array}
\usepackage{theorem}
\usepackage{bm}
\usepackage{braket}
\usepackage{extarrows}
\usepackage{booktabs}
\usepackage{algorithm,algorithmic}
\usepackage{tikz}
\usepackage{color}
\usepackage{xcolor}
\usepackage{subcaption}
\usepackage{multirow}
\usepackage{xr-hyper}
\usepackage{hyperref}
\usepackage[capitalise,noabbrev]{cleveref}

\makeatletter
\newcommand*{\addFileDependency}[1]{
  \typeout{(#1)}
  \@addtofilelist{#1}
  \IfFileExists{#1}{}{\typeout{No file #1.}}
}
\makeatother

\usetikzlibrary{positioning}

\crefname{equation}{}{}

\theoremstyle{plain}
\newtheorem{theorem}{Theorem}

\newtheorem{proof}{Proof}

\newtheorem{thm}{Theorem}[section]
\newtheorem{claim}{Claim}

\newtheorem{prop}{Proposition}[section]

\newtheorem{defn}{Definition}[section]

\newcommand{\norm}[1]{\left\| #1 \right\|}

\newcommand{\card}[1]{\left| #1 \right|}
\newcommand{\paren}[1]{\left( #1 \right)}
\newcommand{\sqbra}[1]{\left[#1 \right]}
\newcommand{\inprod}[2]{\left\langle #1,\ #2 \right\rangle}
\newcommand{\R}{\mathbb{R}}
\newcommand{\N}{\mathbb{N}}
\newcommand{\E}{\mathbb{E}}

\newcommand{\ZZ}{\Sigma}

\author{%
  Akinori ~Tanaka
    \\
  RIKEN AIP, Keio University\\
  \texttt{akinori.tanaka@riken.jp} \\
   \And
   Akiyoshi Sannai \\
   RIKEN AIP, Keio University\\
   \texttt{akiyoshi.sannai@riken.jp} \\
   \AND
   Ken Kobayashi \\
   Fujitsu Laboratories LTD., RIKEN AIP, Tokyo Tech\\
   \texttt{ken-kobayashi@fujitsu.com} \\
   \And
   Naoki Hamada \\
   Fujitsu Laboratories LTD., RIKEN AIP\\
   \texttt{hamada-naoki@fujitsu.com} \\
}

\title{Asymptotic Risk of B\'ezier Simplex Fitting}
\begin{document}
\maketitle

\begin{abstract}
The B\'ezier simplex fitting is a novel data modeling technique which exploits geometric structures of data to approximate the Pareto front of multi-objective optimization problems.
There are two fitting methods based on different sampling strategies.
The \emph{inductive skeleton fitting} employs a stratified subsampling from each skeleton of a simplex, whereas the \emph{all-at-once fitting} uses a non-stratified sampling which treats a simplex as a whole.
In this paper, we analyze the asymptotic risks of those B\'ezier simplex fitting methods and derive the optimal subsample ratio for the inductive skeleton fitting.
It is shown that the inductive skeleton fitting with the optimal ratio has a smaller risk when the degree of a B\'ezier simplex is less than three.
Those results are verified numerically under small to moderate sample sizes.
In addition, we provide two complementary applications of our theory: a generalized location problem and a multi-objective hyper-parameter tuning of the group lasso.
The former can be represented by a B\'ezier simplex of degree two where the inductive skeleton fitting outperforms.
The latter can be represented by a B\'ezier simplex of degree three where the all-at-once fitting gets an advantage.
\end{abstract}

\section{Introduction}
Given functions $f_1,\dots,f_M: X \to \R$ on a subset $X$ of the Euclidean space $\R^N$, consider the multi-objective optimization problem
\[
\text{minimize } f(x) := (f_1(x), \dots, f_M(x)) \text{ subject to } x \in X (\subseteq \R^N)
\]
with respect to the Pareto ordering: $x \prec y \xLeftrightarrow{\mathrm{def}} \forall i \sqbra{f_i(x) \leq f_i(y)} \land \exists j \sqbra{f_j(x) < f_j(y)}$.
The goal is to find the \emph{Pareto set} and its image, called the \emph{Pareto front}, which are denoted by
\[
X^*(f) := \Set{x \in X | \forall y \in X \sqbra{y \not\prec x}}
\quad \text{and} \quad
f(X^*(f)) := \Set{f(x) \in \R^M | x \in X^*(f)},
\]
respectively.
Since most numerical optimization approaches give a finite number of points as an approximation of those objects (e.g., goal programming \cite{Miettinen1999,Eichfelder2008}, evolutionary computation \cite{Deb2001,Zhang2007,Deb2014}, homotopy methods \cite{Hillermeier2001,Harada2007}, Bayesian optimization \cite{Hernandez-Lobato2016,Yang2019}), the complete shapes of them are usually not revealed.
To amplify the knowledge extracted from their point approximations, we consider in this paper a fitting problem of the Pareto set and front.

It is known that those objects often have skeleton structures that can be used to enhance fitting accuracy.
An $M$-objective problem is \emph{simplicial} if the Pareto set and front are homeomorphic to an $(M-1)$-dimensional simplex and each $(m-1)$-dimensional subsimplex corresponds to the Pareto set of an $m$-objective subproblem for all $0 \le m \le M$ (see \cite{Hamada2019} for precise definition and examples).
There are a lot of practical problems being simplicial: location problems \cite{Kuhn1967} and a phenotypic divergence model in evolutionary biology \cite{Shoval2012} are shown to be simplicial, and an airplane design \cite{Mastroddi2013} and a hydrologic modeling \cite{Vrugt2003} hold numerical solutions which imply those problems are simplicial.
The Pareto set and front of any simplicial problem can be approximated with arbitrary accuracy by a B\'ezier simplex of an appropriate degree~\cite{Kobayashi2019}.
There are two fitting algorithms for B\'ezier simplices: the all-at-once fitting is a na\"ive extension of Borges-Pastva algorithm for B\'ezier curves~\cite{Borges2002}, and the inductive skeleton fitting~\cite{Kobayashi2019} exploits the skeleton structure of simplicial problems discussed above.

An important problem class which is (generically) simplicial is the strongly convex problem.
It has been shown that many practical problems can be considered as strongly convex via appropriate transformations preserving the essential problem structure, i.e., the Pareto ordering and the topology~\cite{Hamada2019}.
For example, the multi-objective location problem \cite{Kuhn1967} can be strongly convex by squaring each objective function.
The resulting problem has a Pareto front that can be represented by a B\'ezier simplex of degree two \cite{Hamada2019}.
As we will show in this paper, the group lasso \cite{Yuan2006} can be reformulated as a simplicial problem.
It has a cubic Pareto front that requires a B\'ezier simplex of degree three.
The same transformation can be applied to a broad range of sparse learning methods, including the (original) lasso \cite{Tibshirani1996}, the fused lasso \cite{Tibshirani2005}, the smooth lasso \cite{Hebiri2011}, and the elastic net \cite{Zou2005}.
Since the required degree is problem-dependent, we need to understand the performance of the two B\'ezier simplex fittings with respect to the degree.

In this paper, we study the asymptotic risk of the two fitting methods of the B\'ezier simplex: the all-at-once fitting and the inductive skeleton fitting, and compare their performance with respect to the degree.

Our contributions are as follows:
\begin{itemize}
\item We have evaluated the asymptotic $\ell_2$-risk, as the sample size tends to infinity, of two B\'ezier simplex fitting methods: the all-at-once fitting and the inductive skeleton fitting.
\item In terms of minimizing the asymptotic risk, we have derived the optimal ratio of subsample sizes for the inductive skeleton fitting.
\item We have shown when the inductive skeleton fitting with optimal ratio outperforms the all-at-once fitting when the degree of a B\'ezier simplex is two, whereas the all-at-once has an advantage at degree three.
\item We have demonstrated that the location problem and the group lasso are transformed into strongly convex problems, and their Pareto fronts are approximated by a B\'ezier simplex, which numerically verifies the asymptotic results.
\end{itemize}

The rest of this paper is organized as follows:
\cref{sec:problem-definition} describes the problem definition.
\Cref{sec:asymptotic-risk} analyzes the asymptotic risks of the all-at-once fitting and the all-at-once fitting.
For the inductive skeleton fitting, the optimal subsample ratio in terms of minimizing the risk is derived.
Those analyses are verified in \cref{sec:numerical-examples} via numerical experiments.
\Cref{sec:conclusion} concludes the paper and addresses future work.

\section{Problem definition}\label{sec:problem-definition}
Let $M$ be a non-negative integer.
The \emph{standard $(M - 1)$-simplex} is defined by
\[
\Delta^{M - 1} = \Set{(w_1, \dots, w_M) \in \R^M | \sum_{m = 1}^M w_m = 1,\ w_m \geq 0}.
\]
For an index set $I \subseteq \set{1, \dots, M}$, we define the \emph{$I$-subsimplex} of $\Delta^{M - 1}$ by $\Delta^I = \set{(w_1, \dots, w_M) \in \Delta^{M - 1} | w_m = 0\ (m \not \in I)}$.
For an integer $0 \leq m \leq M$, the \emph{$(m - 1)$-skeleton} of $\Delta^{M - 1}$ is defined by
\[
\Delta^{(m-1)} = \bigcup_{I \subseteq \set{1, \dots, M} \text{ s.t. } \card{I}=m} \Delta^{M-1}_I.
\]

\subsection{B\'ezier simplex and its fitting methods}\label{sec:bezier-simplex}
Let $\N$ be the set of non-negative integers (including zero!) and $M, D \in \N$.
We denote a simplex lattice by $\N_D^M := \set{(d_1,\dots,d_M) \in \N^M | \sum_{m=1}^M d_m = D}$.
Given the \emph{control points} $\bm p_{\bm d} \in \R^L$ $(\bm d \in \N_D^M)$, an \emph{$(M - 1)$-B\'ezier simplex of degree $D$} is a mapping $\bm b(\bm t): \Delta^{M-1}\to\R^L$ defined by
\begin{equation}\label{eqn:bezier-simplex}
\bm b(\bm t) := \sum_{\bm d\in\N_D^M} \binom{D}{\bm d} \bm t^{\bm d} \bm p_{\bm d}
\end{equation}
where $\binom{D}{\bm d} := \frac{D!}{d_1! d_2! \cdots d_M!}$, and for each $\bm t := (t_1, \dots, t_M) \in \R^M$ and $\bm d := (d_1, \dots, d_M) \in \N^M$, $\bm t^{\bm d}$ is a monomial $t^{d_1}_1 t^{d_2}_2 \cdots t^{d_M}_M$.

Kobayashi \text{et al.} \cite{Kobayashi2019} proposed two B\'ezier simplex fitting algorithms: the all-at-once fitting and the inductive skeleton fitting.
They are different in not only fitting algorithm but also sampling strategy.
The all-at-once fitting requires a training set $S_N:=\set{(\bm t_n, \bm x_n) \in \Delta^{M-1} \times \R^L | n = 1, \dots, N}$ and adjusts all control points at once by minimizing the ordinary least square loss: $\frac{1}{N}\sum_{n=1}^N \norm{\bm x_n - \bm b(\bm t_n^{(m)})}^2$.

The inductive skeleton fitting, on the other hand, requires skeleton-wise sampled training sets $S_{N^{(m)}} := \set{(\bm t^{(m)}_n, \bm x^{(m)}_n) \in \Delta^{(m)} \times \R^L | n = 1, \dots, N^{(m)}}~(m=0,\ldots,M-1)$.
It also divides control points as $\bm p_{\bm d}^{(m)}$ such that $\bm d$ has $m+1$ non-zero elements.
Such $\bm p_{\bm d}^{(m)}$ determine $m$-skeleton of a B\'ezier simplex.
The inductive skeleton fitting inductively adjusts $\bm p_{\bm d}^{(m)}$ from $m = 0$ to $M - 1$ by minimizing the ordinary least square loss of the $m$-skeleton $\frac{1}{N^{(m)}}\sum_{n=1}^{N^{(m)}} \norm{\bm x^{(m)}_n - \bm b(\bm t_n)}^2$.

\subsection{The \texorpdfstring{$\ell_2$}{l2}-risk}
The fitting problem considered in this paper is as follows.
The sample is taken from an unknown B\'ezier simplex $\bm b(\bm t): \Delta^{M-1} \to \R^L$ with additive Gaussian noise $\bm \varepsilon \sim N(\bm 0, \sigma^2 \bm I)$, that is, $\bm x = \bm b(\bm t) + \bm \varepsilon$.
For the all-at-once fitting, $S_N = \set{(\bm t_n, \bm x_n)}$ follows the uniform distribution on the domain of the B\'ezier simplex: $\bm t_n \sim U(\Delta^{M-1})$ and $\bm x_n = \bm b(\bm t_n) + \bm \varepsilon_n$.
For the inductive skeleton fitting, $S_{N^{(m)}} = \set{(\bm t_n^{(m)}, \bm x_n^{(m)})}$ follows the uniform distribution on the $m$-skeleton of the domain of the B\'ezier simplex: $\bm t_n^{(m)} \sim U(\Delta^{(m)})$ and $\bm x_n^{(m)} = \bm b(\bm t_n^{(m)}) + \bm \varepsilon_n^{(m)}$.
A B\'ezier simplex estimated from $S_N$ is denoted by $\bm{\hat b}(\bm t | S_N)$.
For both method, we asymptotically evaluate the $\ell_2$-risk below as $N \to \infty$.
\begin{equation}\label{eqn:risk-def}
R_N := \E_{S_N}\sqbra{\E_{\bm t \sim U(\Delta^{M-1})} \norm{\bm b(\bm t) - \hat{\bm b}(\bm t | S_N)}^2}.
\end{equation}
Here, we put $S_N = S_{N^{(0)}} \cup \dots \cup S_{N^{(M-1)}}$ for the inductive skeleton fitting.

\section{Asymptotic risk of B\'ezier simplex fitting}\label{sec:asymptotic-risk}
To calculate the risk of each fitting scheme, let us first focus on the fact: the summation/subtraction of two B\'ezier simplices is also B\'ezier simplex.
In the definition \cref{eqn:risk}, we have $\bm b(\bm t) - \hat{\bm b}(\bm t | S_N)$ which measures the difference between the target B\'ezier simplex and the model B\'ezier simplex.
By using the above fact, $\bm b(\bm t) - \hat{\bm b}(\bm t | S_N)$ is also B\'ezier simplex.
Let us call its control point as $\bm p'$, and consider the following matrix,
\begin{align}
\bm P
=
\begin{bmatrix}
(\bm p_1')_1 &
(\bm p_1')_2 &
\cdots &
(\bm p_1')_L
\\
(\bm p_2')_1 &
(\bm p_2')_2 &
\cdots &
(\bm p_2')_L
\\
\quad \vdots&
\quad \vdots &
\ddots &
\quad \vdots
\\
(\bm p_{\card{\N_D^M}}')_1 &
(\bm p_{\card{\N_D^M}}')_2 &
\cdots &
(\bm p_{\card{\N_D^M}}')_L
\end{bmatrix}
,
\end{align}
where $(\bm p_A)_{l}$ means the $l$-th component of the $A$-th control point vector $\bm p_A$.
The asymptotic rick can be calculated by the following theorem.
\begin{theorem}\label{thm:risk}
The risk of the B\'ezier simplex fitting can be represented by
\begin{align}
R_N
=
\sum_{\bm d_A, \bm d_B \in \mathbb{N}_D^M}
\ZZ_{AB}
\E_{S_N} \sqbra{({\bm P} {\bm P}^\top)_{AB}}
,
\label{eqn:risk}
\end{align}
where the matrix $\ZZ$ is defined by
\begin{align}
\ZZ_{AB}
=
\frac{(2D)!(M-1)!}{(2D+M-1)!}
  \binom{D}{\bm d_A}
  \binom{D}{\bm d_B}
  \binom{2D}{\bm d_A + \bm d_B}^{-1}
  \label{eqn:ZZ-def}
\end{align}
\end{theorem}
The proof is provided in the supplementary materials (\cref{sec:proof-of-main-theorem}).
Once the set of parameters of the system, including the simplex dimension $M$, degree of the B\'esier simplex $D$, the dimension of the target data $L$, the amplitude of the noise $\sigma$, is fixed, the equation \cref{eqn:risk} says that the asymptotic value of this risk function depends only on how we choose the matrix $\bm P$.
We calculate the asymptotic form of the risk \cref{eqn:risk} with $\bm P$ determined from the all-at-once (AAO) fitting and the inductive-skeleton (ISK) fitting.
We call them $\bm P_\mathrm{AAO}$ and $\bm P_\mathrm{ISK}$ respectively.

\subsection{All-at-once fitting}\label{sec:all-at-once-fitting}
\paragraph{Samples and determined control points}
Let us recall the sample $S_N$ consists of elements $(\bm t_n, \bm x_n) \in \Delta^{M-1} \times \R^L$ with $\bm x_n = \bm b(\bm t_n) + \bm \varepsilon_n$ and $\bm t_n \sim U(\Delta^{M-1})$, $\bm \varepsilon_n \sim N(0, \sigma^2 \bm I)$ for $n = 1, \dots, N$.
The matrix $\bm P_\mathrm{AAO}$ is determined by minimization of the OLS error below
\begin{align}
\frac 1 N \sum_{n=1}^N \norm{\bm x_n - \hat{\bm b}(\bm t_n)}^2
&= \frac 1 N \norm{ \bm Z \bm P + \bm Y }_{\mathrm F}^2, \label{eqn:rss}
\end{align}
where $\norm{\cdot}_\mathrm{F}$ means the Frobenius norm and
$\bm z_n = [\text{a vector with component } \binom{D}{\bm d}\bm t_n^{\bm d}~(\bm d\in \N_D^M)] \in \R^{\card{\N_D^M}}$, $\bm Z = \sqbra{\bm z_1 \bm z_2 \cdots \bm z_N}^\top \in \R^{N \times \card{\N_D^M}}$, $\bm Y = \sqbra{\bm \varepsilon_1 \bm \varepsilon_2\cdots \bm \varepsilon_N}^\top \in \R^{N \times L}$.
In this notation, the optimum takes well known form: $\bm P_\mathrm{AAO} = -\paren{\bm Z^\top \bm Z}^{-1}\bm Z^\top \bm Y$.
Note that the regularity of the matrix $\bm Z^\top \bm Z$ is guaranteed by taking a sufficiently large number of samples $S_N$, or more precisely $\set{\bm t_n}_{n = 1, \dots, N}$.

\paragraph{Calculation of the asymptotics}
To calculate the risk asymptotics \cref{eqn:risk} in the all-at-once fitting, we need to calculate asymptotic values of expectation values of the matrix $\bm P_\mathrm{AAO} \bm P_\mathrm{AAO}^\top$ over $S_N$.
The first observation is that the contribution from the noise matrix $\bm Y$ is only located at the middle of the sequence of matrix product $\bm P_\mathrm{AAO} \bm P_\mathrm{AAO}^\top = \paren{\bm Z^\top \bm Z}^{-1}\bm Z^\top \bm Y \bm Y^\top \bm Z \paren{\bm Z^\top \bm Z}^{-1}$.
If $\bm Y \bm Y^\top \propto 1_N$, the calculation reduces very simple form.
In fact, we can perform it by decomposing $\E_{S_N}$ to $\E_{\bm t_n} \E_{\bm \varepsilon_n}$.
After taking expectation value $\E_{\bm \varepsilon_n}$,
we get
\begin{align}
\E_{S_N} \sqbra{\bm P_\mathrm{AAO} \bm P_\mathrm{AAO}^\top}
&= 
\sigma^2 L \cdot
\E_{\bm t_n} \sqbra{
\paren{\bm Z^\top \bm Z}^{-1}
}
.
\label{eqn:P_AAO-P_AAO}
\end{align}
The prefactor $\sigma^2 L$ results from taking expectation over the noise.
Here, $L$ is the dimension of the space control points lived in.
Now, the only remaining task is the estimation of the asymptotic behavior of the matrix $(\bm Z^\top \bm Z)$.
A key observation is that each element of this matrix is an average over the samples $\bm t_n$:
\begin{align}
\frac{1}{N}
\paren{\bm Z^\top \bm Z}_{AB}
=
\binom{D}{\bm d_A}
\binom{D}{\bm d_B}
\sum_{n=1}^N
\frac{1}{N}
\bm t_n^{\bm d_A + \bm d_B}
.
\end{align}
In fact, it converges to the matrix $\ZZ_{AB}$ defined in \cref{eqn:ZZ-def} as $N \to \infty$.
Therefore, by using the law of large numbers, we can get
$
\paren{\bm Z^\top \bm Z}_{AB}
\overset{p}{\to}
N \ZZ_{AB}.
$
To substitute it to \cref{eqn:P_AAO-P_AAO}, we need to guarantee $\ZZ_{AB}$ has the inverse matrix, i.e.
\begin{theorem}\label{thm:sigma-is-non-singular}
For any $D, M$, the matrix $\ZZ_{AB}$ in \cref{eqn:ZZ-def} is non-singular.
\end{theorem}
The proof is given in the supplementary materials (\cref{sec:proof-of-regularity}).
Just by replacing $(\bm Z^T \bm Z)$ to $N \ZZ$, we arrive at the asymptotic form of the risk.
In addition to it, we can further simplify the result by using: $\sum_{AB} \ZZ_{AB}
\ZZ^{-1}_{AB} = \ _{D+M-1} C_D$, which is relatively easy to show (see \cref{sec:derivation-of-sum} in the supplementary materials).

In summary, our formula for the asymptotic form of the risk for the all-at-once fitting is
\begin{align}
R_N \overset{p}{\to}
\frac{\sigma^2 L}{N}
\sum_{A, B}
\ZZ_{AB}
\ZZ^{-1}_{AB}
=
\frac{\sigma^2 L}{N}
\ _{D+M-1} C_D \quad \text{as $N \to \infty$}.
\label{eqn:risk-all-at-once}
\end{align}

\subsection{Inductive skeleton fitting}\label{sec:inductive-skeleton-fitting}
Before showing the asymptotic form of the risk for the inductive skeleton fitting, it would be better to introduce some notations here.
To treat subsimplices of the simplex $\Delta^{M-1}$, it is useful to notice that there is a one-to-one correspondence between a subsimplex and $M$-dimensional binary vector:
\begin{align}
\text{a subsimplex of $\Delta^{M-1}$}
\quad
\Longleftrightarrow
\quad
\bm I = [I_1, I_2, \dots, I_M], \quad
I_i \in \Set{0, 1}, \quad
\bm I \neq \bm 0.
\end{align}
In this notation, $\Delta^{M-1}$ itself is identified to $[1,1, \dots, 1]$.
The sum $\card{\bm I} = \sum_{i=1,2, \dots, M} I_i$ provides the dimension + 1 of the corresponding subsimplex.
We call a subsimplex indexed by $\bm I$ as $\Delta^{\bm I}$ from now on.
In addition, it is useful to define notation for the set of all $(m-1)$-dimensional subsimplices:
\begin{align}
(m) = \text{all $(m)$-dimensional subsimplices of $\Delta^{M-1}$}
=
\cup_{\card{\bm I} = m+1} \Delta^{\bm I}
,
\end{align}
and we call corresponding control point submatrices as $\bm P^{(m)}$.

\paragraph{Samples and determined control points}
In this notation, we can state that the inductive skeleton fitting is an inductive procedure of determining control points matrices $\bm P^{(m)}$ from low $m=0, 1, \dots, M-1$.
Suppose all $\set{\bm P^{(k)}}_{k < m}$ are already fixed and the samples on $\cup_{\card{\bm I} = m} \Delta^{\bm I}$,
$S_{N^{(m)}} = \set{(\bm t_1^{(m)}, \bm x_1^{(m)}), \dots, (\bm t_{N^{(m)}}^{(m)}, \bm x_{N^{(m)}}^{(m)})}$ are provided from
${\bm t}_n^{(m)} \sim U(\cup_{\card{\bm I} = m+1} \Delta^{\bm I})$.
The $m$-th submatrix $\bm P^{(m)}$ is determined by minimizing the OLS error
\begin{align}
&
\frac{1}{N^{(m)}}
\sum_{n = 1}^{N^{(m)}}
\norm{
\bm x_n^{(m)}
-
\hat{\bm b}(\bm t_n^{(m)})
}^2\label{eqn:OLS}
\end{align}
Note that there is no need to take any control point on $\Delta^{\bm J}$ with $\card{\bm J} > m$ into account because each $\bm t_n^{(m)}$ is on $\Delta^{\bm I}$ with $\card{\bm I} = m+1$ and there is no contribution to $\hat{\bm b}(\bm t_n^{(m)})$ from such higher dimensional control point.
In addition, we regard lower dimensional control points already fixed, so the net objective control points are ones included in $\bm P^{(m)}$.
By repeating similar procedure done in the all-at-once fitting, we can conclude $\bm P^{(m)}$ is determined as
\begin{align}
{\bm P}_\mathrm{OLS}^{(m)}
= 
-
[({\bm Z}^{(m)})^\top {\bm Z}^{(m)}]^{-1}
({\bm Z}^{(m)})^\top
\paren{
{\bm Y}^{(m)}
+
\sum_{k < m}
{\bm Z}^{(m) [k]}
{\bm P}_\mathrm{OLS}^{(k)}
}
\label{eqn:P_OLS}
\end{align}
where
$
{\bm z}_{n}^{(m)[k]}
=
\sqbra{\text{a vector with component }
({\bm z}_{n}^{(m)})^{{\bm d}^{(k)}}
}$,
$
{\bm Z}^{(m) [k]}
=
\sqbra{{\bm z}_1^{(m) [k]}
 \bm{z}_2^{(m) [k]}
 \cdots
 {\bm z}_{N^{(m)}}^{(m) [k]}}^\top$,
$
{\bm Y}^{(m)}
=
\sqbra{{\bm \varepsilon}_1^{(m)} \bm{\varepsilon}_2^{(m)} \cdots {\bm \varepsilon}_{N^{(m)}}^{(m)}}^\top
.
$

\paragraph{Calculation of the asymptotics}
We get ${\bm P_\mathrm{ISK}} {\bm P_\mathrm{ISK}}^\top
=
\oplus_{i, j=0}^{M-1}
{\bm P}^{(i)}_\mathrm{OLS}
({\bm P}^{(j)}_\mathrm{OLS})^\top
$ which we need to compute the risk \cref{eqn:risk}.
%
As one might notice, the risk for the inductive-skeleton fitting depends on each number of $(m)$-dimensional subsamples $N^{(m)}$.
We will determine the best combination of $N^{(m)}$ constrained on $\sum_m N^{(m)} = N$ later.
Here, we treat the risk depending not $N$ but every $N^{(m)}$ and call it as
$
R_{N^{(0)}, N^{(1)}, \dots, N^{(M-1)}}
$.
To calculate $\E_{S_N} [\bm P_{ISK} \bm P_{ISK}^\top]$, we again take expectation over noise.
Thanks to $\E [\bm Y^{(m)} (\bm Y^{(n)})^\top ] = \sigma^2 L \bm 1_{N^{(m)}}$ or $\bm 0$ depending on $m=n$ or not, and the central limit with respect to $\bm z_n$, one can get
\begin{align}
&\E_{ S_N }
\sqbra{
{\bm P}^{(i)}_\mathrm{OLS}
({\bm P}^{(j)}_\mathrm{OLS})^\top
}
\notag \\
&\overset{p}{\to}
\sigma^2 L
\sum_{
\substack{
m \leq i
\\
m \leq j
}}
\sum_{
\substack{
m \leq k_1 < \dots < k_{\heartsuit} < i
\\
m \leq l_1 < \dots < l_{\spadesuit} < j
}
}
\frac{(-1)^{\heartsuit + \spadesuit}}{N^{(m)}}
{\bm \Lambda}_{(i)}
{\bm \Lambda}^{(i)[k_{\heartsuit}]}
{\bm \Lambda}_{(k_{\heartsuit} )}
\cdots
{\bm \Lambda}^{(k_1)[m]}
{\bm \Lambda}_{(m)}
{\bm \Lambda}^{[m](l_1)}
\cdots
{\bm \Lambda}_{(l_{\spadesuit} )}
{\bm \Lambda}^{[l_{\spadesuit}](j)}
{\bm \Lambda}_{(j)}
\label{eqn:P_OLS-P_OLS}
\end{align}
after substituting the recursive formula \cref{eqn:P_OLS} repeatedly, where
\begin{align}
&({\bm \Lambda}^{(m)[k]})_{ \bm d ^{(m)} \bm d ^{(k)} }
= 
\frac{(m-1)!}{\ _{M} C_m}
\begin{pmatrix}
D \\
{\bm d}^{(m)}
\end{pmatrix}
\begin{pmatrix}
D \\
{\bm d}^{(k)}
\end{pmatrix}
\sum_{ \card{I} = m }
\frac{
\delta_{{\bm I}, (\bm d ^{(m)} +  \bm d ^{(k)} )_{01} }
\prod_{I_i = 1}( \bm d ^{(m)} +  \bm d ^{(k)}  )_i !}{
[\sum_{I_i = 1} ( \bm d ^{(m)} +  \bm d ^{(k)}  )_i
+ m - 1]!
},
\notag
\\
&
{\bm \Lambda}^{[k](m)}
=
{\bm \Lambda}^{(m)[k]},
\quad
{\bm \Lambda}_{(m)}
=
({\bm \Lambda}^{(m)[m]})^{-1}
\label{eqn:Lamb}
\end{align}
and
\begin{align}
\delta_{{\bm I}, (\bm d ^{(m)} +  \bm d ^{(k)} )_{01} }
&=
\left\{ \begin{array}{ll}
1 & ( {\bm I} = (\bm d ^{(m)} +  \bm d ^{(k)} )_{01}  ) \\
0 & \text{otherwise}\\
\end{array} \right.
\end{align}
For the complete derivation, see \cref{sec:risk-derivation} in the supplementary materials.
We stop here and leave to get a closed formula for the asymptotics of the risk on inductive-skeleton fitting as future work.
Instead, we calculate the asymptotic risk numerically by using \cref{eqn:P_OLS-P_OLS,eqn:risk} and obtain \cref{tab:risk-inductive}.
\begin{table}[H]
\centering
\caption{Numerically computed asymptotic risks of the inductive skeleton fitting ($M$: dimension of B\'ezier simplex, $D$: degree of B\'ezier simplex, $N^{(m)}$: sample size of $(m)$-skeleton).}\label{tab:risk-inductive}
\begin{tabular}{l|ll} \toprule
$R_{N^{(0)}, N^{(1)}, \dots}$ &
$D = 2$ &
$D = 3$
\\ \midrule
$M = 2$ &
$1.0 / N^{(1)} + 0.5 / N^{(0)}$ &
$2.0 / N^{(1)} + 0.2666 / N^{(0)}$
\\
$M = 3$ &
$3.0 / N^{(1)} + 0.375 / N^{(0)}$ &
$1.0 / N^{(2)} + 3.535 / N^{(1)} + 0.1464 / N^{(0)}$
\\
$M = 4$ &
$5.142 / N^{(1)} + 0.4571 / N^{(0)}$ &
$5.333 / N^{(2)} + 4.714 / N^{(1)} + 0.1650 / N^{(0)}$
\\
$M = 5$ &
$7.142 / N^{(1)} + 0.625 / N^{(0)}$ &
$13.33 / N^{(2)} + 6.666 / N^{(1)} + 0.2083 / N^{(0)}$
\\
$M = 6$ &
$8.928 / N^{(1)} + 0.8214 / N^{(0)}$ &
$24.24 / N^{(2)} + 9.740 / N^{(1)} + 0.2575 / N^{(0)}$
\\
$M = 7$ &
$10.5 / N^{(1)} + 1.020 / N^{(0)}$ &
$37.12 / N^{(2)} + 13.84 / N^{(1)} + 0.3119 / N^{(0)}$
\\
$M = 8$ &
$11.87 / N^{(1)} + 1.212 / N^{(0)}$ &
$51.17 / N^{(2)} + 18.73 / N^{(1)} + 0.3723 / N^{(0)}$
\\ \bottomrule
\end{tabular}
\end{table}

\subsection{All-at-once vs Inductive skeleton}\label{sec:all-vs-inductive}
\Cref{tab:risk-inductive} tells the risk of the inductive skeleton fitting depends on subsample sizes $N^{(m)}$.
Given total sample size $N$, we can minimize the risk by finding the optimally-decoupled subsample sizes:
\begin{align}
    R_N := \min_{N^{(0)}, \dots, N^{(M-1)}} \Set{R_{N^{(0)}, \dots, N^{(M-1)}}} \text{ subject to } \sum_{m=0}^{M-1} N^{(m)} = N.
\end{align}
We calculated optimal risks for all cases shown in \cref{tab:risk-inductive} and compared them to the risks of the all-at-once fitting.
\Cref{tab:risk-comparison} shows the results.
\begin{table}[H]
\centering
\caption{Comparison of asymptotic risks of the all-at-once $R_N^\mathrm{AAO}$ vs the inductive skeleton with the optimal subsample ratio $R_N^\mathrm{ISK}$ ($M$: dimension of B\'ezier simplex, $D$: degree of B\'ezier simplex, $N$: sample size). The winner is shown in bold.}\label{tab:risk-comparison}
\begin{tabular}{l|ll|ll} \toprule
& \multicolumn{2}{c|}{$D = 2$} & \multicolumn{2}{c}{$D = 3$} \\
& $R_{N}^\mathrm{AAO}$ & $R_N^\mathrm{ISK}$
& $R_{N}^\mathrm{AAO}$ & $R_N^\mathrm{ISK}$
\\ \midrule
$M = 2$ &
$3.0 / N$ &
$\bm{2.91421} / N$ &
$4.0 / N$ &
$\bm{3.72726} / N$
\\
$M = 3$ &
$6.0 / N$ &
$\bm{5.49632} / N$ &
$\bm{10.0} / N$ &
$10.6472 / N$
\\
$M = 4$ &
$10.0 / N$ &
$\bm{8.66660} / N$ &
$\bm{20.0} / N$ &
$23.8821 / N$
\\
$M = 5$ &
$15.0 / N$ &
$\bm{11.9936} / N$ &
$\bm{35.0} / N$ &
$44.7548 / N$
\\
$M = 6$ &
$21.0 / N$ &
$\bm{15.1663} / N$ &
$\bm{56.0} / N$ &
$73.1387 / N$
\\
$M = 7$ &
$28.0 / N$ &
$\bm{18.0687} / N$ &
$\bm{84.0} / N$ &
$107.570 / N$
\\
$M = 8$ &
$36.0 / N$ &
$\bm{20.6799} / N$ &
$\bm{120.0} / N$ &
$146.206 / N$
\\ \bottomrule
\end{tabular}
\end{table}

As one can see, the optimum inductive skeleton fitting outperforms the all-at-once fitting in $D = 2$, but it is not always correct in $D = 3$.
On $D = 2$, in fact, we can show that the minimum value of the inductive skeleton always less than the asymptotic risk of the corresponding all-at-one fitting.



\section{Numerical examples}\label{sec:numerical-examples}
We examine the empirical performances of the all-at-once fitting and the inductive skeleton fitting and verify the asymptotic risks derived in \cref{sec:all-at-once-fitting,sec:inductive-skeleton-fitting} over synthetic instances and multi-objective optimization instances.
Experiment programs were implemented in Python 3.7.1 and run on a Windows 7 PC with an Intel Core i7-4790CPU (3.60 GHz) and 16 GB RAM.
All experiments are reproducible by the source code and dependent libraries provided in the supplementary materials.

\subsection{Synthetic instances}\label{sec:synthetic-instances}
To verify the asymptotic risks derived in \cref{sec:all-at-once-fitting,sec:inductive-skeleton-fitting}, we consider the fitting problem where the true B\'ezier simplex $\bm b(\bm{t})~(\bm t \in \Delta^{M-1})$ is an $(M-1)$-dimensional unit simplex on $\R^L$, and randomly generate $N$ training points $\set{(\bm t_n, \bm x_n)}_{n = 1}^N$ as $\bm x_n = \bm b(\bm t_n) + \bm{\varepsilon}_n~(\bm \varepsilon_n \sim N(\bm 0, 0.1^2 \bm I))$.
This synthetic instance is parameterized by a tuple $(L, M, N)$.
The detailed data generation processes are shown in the supplementary materials (\cref{sec:numerical-experiments}).

In this experiment, we estimated the B\'ezier simplex with degree $D = 2$ and 3, and compared the following three fitting methods:
\begin{description}
    \item[all-at-once] the all-at-once fitting (\cref{sec:all-at-once-fitting});
    \item[inductive skeleton (non-optimal)] the inductive skeleton fitting (\cref{sec:inductive-skeleton-fitting}) with $N^{(0)} = \dots = N^{(M - 1)} = N / M$, which does not provide the optimal value of the risk \cref{tab:risk-inductive};
    \item[inductive skeleton (optimal)] the inductive skeleton fitting (\cref{sec:inductive-skeleton-fitting}) where $N^{(0)}, \dots, N^{(M - 1)}$ are determined by minimizing the risk \cref{tab:risk-inductive} under the constraints $\sum_{m = 0}^{M-1} N^{(m)} = N$ and $N^{(m)}\geq 0~(m = 0, \dots, M-1)$. The actual sample size $N^{(m)}$ for each $(D, M)$ are shown in \cref{sec:numerical-experiments} (\cref{tab:optimal-subsample-ratio}).
\end{description}
When we calculated an approximation of the expected risk for each method, we randomly chose other 10000 parameters $\set{\bm{\hat t}_n}_{n = 1}^{10000}$ from $U(\Delta^{M - 1})$ as a test set and measured the mean squared error, $\mathrm{MSE} := \frac{1}{10000} \sum_{n = 1}^{10000} \norm{\bm b(\bm{\hat{t}}_n) - \bm{\hat{b}}(\bm{\hat{t}}_n)}^2$, where $\bm{\hat{b}}$ is the estimated B\'ezier simplex.
We ran 20 trials and measured MSEs for each $(L, M, N)$ with $D \in \set{2, 3}$.

Owing to space limitations, we only present typical results here. The remaining results are provided in the supplementary materials (\cref{sec:numerical-experiments}).
\Cref{fig:MSE-vs-N} shows box plots of MSEs over 20 trials and our theoretical risks \cref{eqn:risk} and \cref{tab:risk-inductive} for each $N \in \set{250, 500, 1000, 2000}$ with $(L, M) = (100, 8)$ and $D \in \set{2, 3}$.
We observe that these figures empirically show that our theoretical risks are correct for both $D = 2$ and 3, and the gap between the actual MSEs and the risks are sufficiently small at $N = 1000$.
For both $D = 2$ and 3, the inductive skeleton (optimal) always achieved lower MSEs than that of the inductive skeleton (non-optimal).
This result suggests the efficiency of minimizing the risk (\cref{tab:risk-comparison}) with respect to the sample size of each dimension.
In addition, the inductive skeleton fitting (optimal) also outperformed the all-at-once fitting in the case of $D = 2$.
This result also supports the discussion described in \cref{sec:all-vs-inductive}.
\begin{figure}[h]
 \begin{minipage}{0.49\hsize}
        \centering
    \includegraphics[width=1\textwidth]{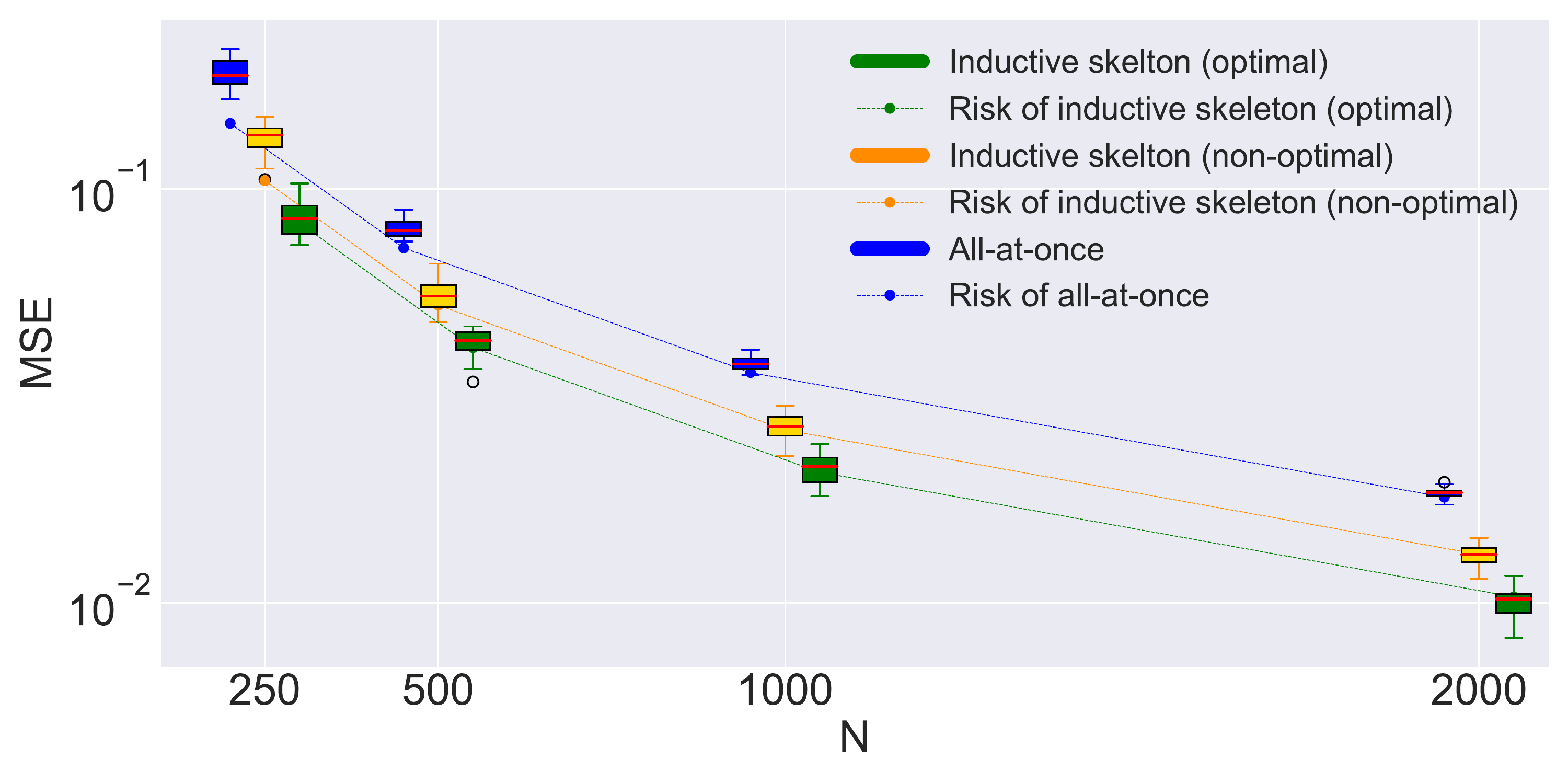}
    \subcaption{D=2}
    \label{fig:MSE-vs-N-D=2}
 \end{minipage}
 \begin{minipage}{0.49\hsize}
        \centering
    \includegraphics[width=1\textwidth]{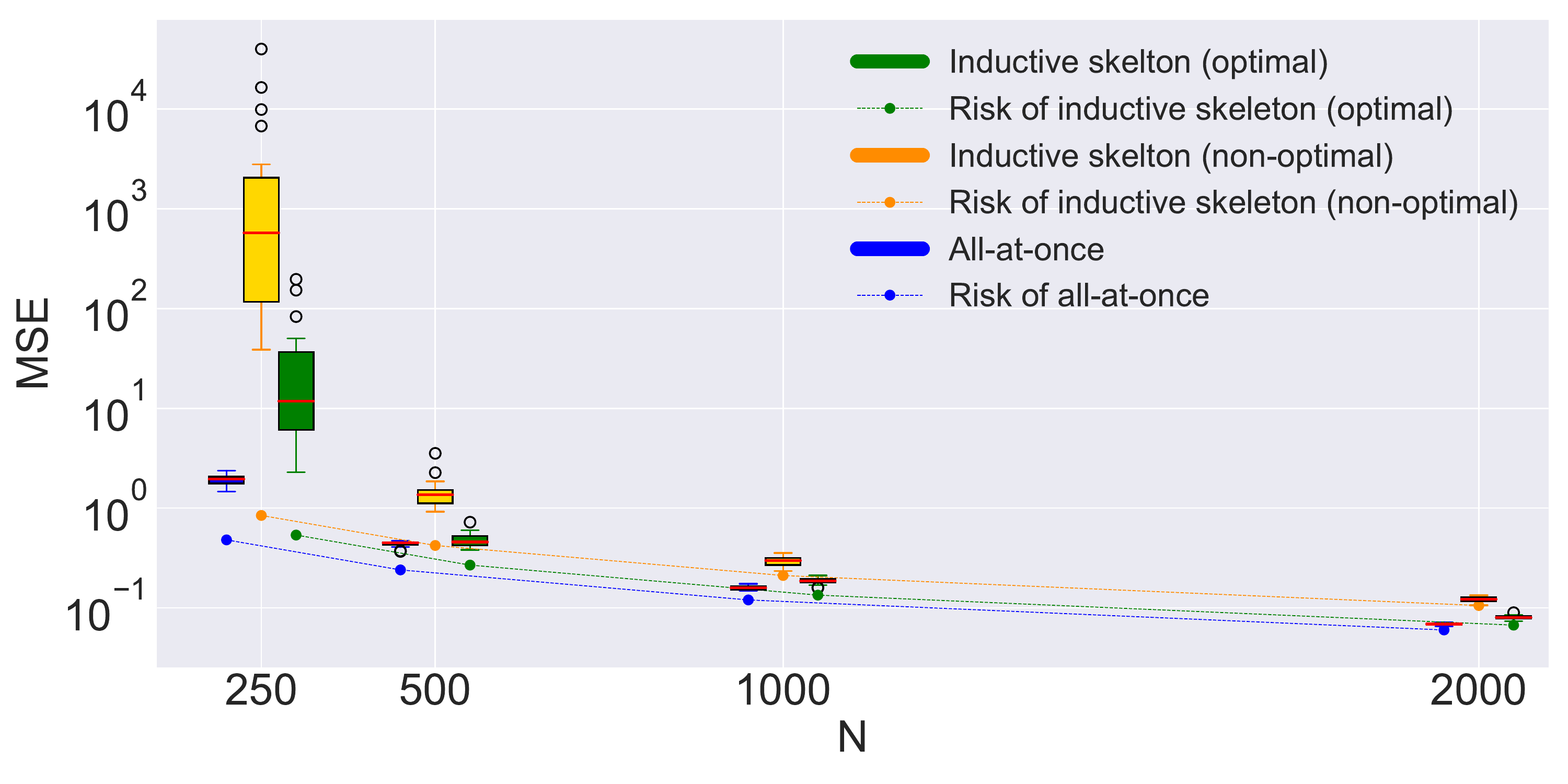}
    \subcaption{D=3}
    \label{fig:MSE-vs-N-D=3}
 \end{minipage}
 \caption{Sample size $N$ vs. MSE with $(L, M) = (100, 8)$ (boxplots over 20 trials and theoretical risks).}
 \label{fig:MSE-vs-N}
\end{figure}

\subsection{Multi-objective optimization instances}\label{sec:MOP-instances}
To investigate our results practically, we provide two complementary multi-objective optimization problem instances: a generalized location problem called \texttt{MED} \cite{Harada2006,Hamada2010} and a multi-objective hyper-parameter tuning of the group lasso \cite{Yuan2006} on the \texttt{Birthwt} dataset \cite{Hosmer1989,Venables2002}.
Both of them are strongly convex three-objective optimization problems and we consider fitting their two-dimensional (that is $M = L = 3$) Pareto fronts by a B\'ezier simplex with degree $D = 2$ and 3. 
For the location problem, its Pareto front can be represented by a B\'ezier simplex with degree $D = 2$.
For the group lasso, on the other hand, the Pareto front cannot be represented with degree $D = 2$ but $D = 3$ (see \cref{sec:Pareto-fronts}).
The detailed description of each problem is shown in \cref{sec:numerical-experiments}.

As we conducted in the previous experiments, we generated a training set and a test set on a Pareto front randomly then fitted a B\'ezier simplex to the training set and evaluated the MSE between the estimated B\'ezier simplex and the test set.
We chose the number of training points to $N = 50$ and 100.
With regard to the test set, the number of sample points is 10000 and 1000 for the location problem and the group lasso respectively.
We repeated experiments 20 times for each $(D, N)$.

For each problem instance and method, the average and the standard deviation of the MSE are shown in \cref{tab:MSE-MOP-instances}.
In \cref{tab:MSE-MOP-instances}, we highlighted the best score of MSE out of all-at-once fitting and inductive skeleton fitting (optimal) and added the results of one-sided Student's t-test with significance level 0.05.

\begin{table}[h]
\caption{MSE (avg.\ $\pm$ s.d.\ over 20 trials) for the location problem and the group lasso. The winners with significance level $p < 0.05$ are shown in bold.}
\label{tab:MSE-MOP-instances}
\begin{minipage}[t]{.5\textwidth}
    \centering
    \subcaption{Location problem}
    \scriptsize
    {\tabcolsep = 1.5mm
    \begin{tabular}{ccll}
    \toprule
$D$&$N$ & \multicolumn{1}{c}{All-at-once} & \multicolumn{1}{c}{Inductive-skeleton (optimal)}\\ \midrule
2&50	&2.855e-04 $\pm$	2.114e-05	&\textbf{2.691e-04 $\pm$	8.541e-06}\\
&100 &2.660e-04 $\pm$	1.227e-05	&\textbf{2.608e-04 $\pm$	5.946e-06}\\ \midrule
3&50	&3.596e-04 $\pm$	7.935e-05	&3.269e-04 $\pm$	3.969e-05\\
&100	&2.810e-04 $\pm$	1.569e-05	&2.796e-04 $\pm$	1.478e-05\\
\bottomrule
    \end{tabular}
    }
    \label{tab:location-problem}
\end{minipage}
\begin{minipage}[t]{.5\textwidth}
    \centering
    \subcaption{Group lasso}
    \scriptsize
    {\tabcolsep = 1.5mm
    \begin{tabular}{ccll}
    \toprule
$D$&$N$ & \multicolumn{1}{c}{All-at-once} & \multicolumn{1}{c}{Inductive-skeleton (optimal)}\\ \midrule
2&50	&\textbf{1.041e-04 $\pm$	 1.614e-05}	&4.966e-04 $\pm$	1.848e-05\\
& 100 &\textbf{8.949e-05 $\pm$	6.083e-06}	&5.020e-04 $\pm$	1.276e-05\\ \midrule
3&50	&\textbf{4.354e-05 $\pm$	 1.526e-05}	&1.206e-04 $\pm$ 	9.440e-06\\
& 100 &\textbf{3.231e-05 $\pm$	8.058e-06}	&1.141e-04 $\pm$	8.200e-06\\
 \bottomrule
    \end{tabular}
    }
    \label{tab:group-lasso}
\end{minipage}
\end{table}

Since the Pareto front of the location problem can be represented by a B\'ezier simplex of $D = 2$ and 3, we expected that the experimental results agree with our analysis discussed in \cref{sec:all-vs-inductive}.
In fact, \cref{tab:location-problem} shows that the inductive skeleton (optimal) outperformed for $D = 2$, which is consistent with our analysis.
For $D = 3$, the difference of MSEs is not significant.
\Cref{tab:risk-comparison} suggests that the difference of the risks between the two methods is very small for $(D, M) = (3, 3)$, and thus we did not observe significant differences of MSEs for $N = 50$ and 100.


In case of the group lasso, on the other hand, \cref{tab:group-lasso} shows that the all-at-once was better for both $D = 2$ and 3, and the differences are all significant.
While our analysis assumes that the target hypersurface to be fitted can be represented by a B\'ezier simplex, the Pareto front of the group lasso cannot for $D = 2$ but for $D = 3$.
Therefore, the results for $D = 2$ does not contradict to our analysis.
Moreover, the results for $D = 3$ that the all-at-once achieved better MSEs accords with our analysis.


From the above results, the validity of the analytic results is confirmed in practical situations.

\section{Conclusion}\label{sec:conclusion}
In this paper, we have shown that the asymptotic $\ell_2$-risk of the two B\'ezier simplex fitting methods developed previously: the all-at-once fitting and the inductive skeleton fitting.
From our risk analysis, the optimal ratio of subsamples for the inductive skeleton fitting has been derived, which is useful for design of experiments to maximize the goodness of fit.
We have discussed that superiority between the two fitting methods depends on the degree of a B\'ezier simplex to be fit: the inductive skeleton fitting with optimally-decoupled subsamples outperforms for degree two whereas the all-at-once fitting becomes the better for degree three, independent of the dimensionality of the B\'ezier simplex and its ambient space.
The above theoretical results have been confirmed via numerical experiments under small to moderate sample sizes.
We have demonstrated two applications of the analytic results in multi-objective optimization: a generalized location problem and a hyper-parameter tuning of the group lasso.

As a remark for future work, we point out two important cases which the current theory does not cover.
The first one is the case discussed in \cref{sec:MOP-instances} that the true surface is not representable by a model.
The second one is presented in the literature \cite{Kobayashi2019}.
When the parameters of a B\'ezier simplex are not given in a sample and to be estimated as well as the control points, the inductive skeleton fitting outperforms the all-at-once fitting even if the B\'ezier simplex is of degree \emph{three}.
We believe that those cases would offer insightful examples to extend the scope of the theory.

\bibliographystyle{plain}

\clearpage
\appendix
\section{Proof of \texorpdfstring{\cref{thm:risk}}{theorem 2}}\label{sec:proof-of-main-theorem}
As we commented in the main body of the paper, we can regard $\bm b (\bm t) - \hat{\bm b} (\bm t | S_N)$ as a new B\'ezier simplex, and we call its control points as $\bm p'_{\bm d}$.
By using it, the risk $R_N = \E_{S_N} \sqbra{\E_{\bm t} \sqbra{\norm{\bm b (\bm t) - \hat{\bm b} (\bm t | S_N)}^2}}$ can be rewritten as
\begin{align}
R_N
&=
\E_{S_N}
\sqbra{
\E_{\bm t}
\sqbra{
\norm{
\sum_{\bm d \in \N_D^M}
\binom{D}{\bm d}
\bm t^{\bm d} \bm p_{\bm d}'
}^2
}
}
\notag \\
&=
\E_{S_N}
\sqbra{
\E_{\bm t}
\sqbra{
\sum_{\bm d_A, \bm d_B \in \N_D^M}
\binom{D}{\bm d_A}
\binom{D}{\bm d_B}
\bm t^{\bm d_A + \bm d_A }
\paren{
\bm p_{\bm d_A}'
\cdot
\bm p_{\bm d_B}'
}
}
}
\notag \\
&= 
\sum_{\bm d_A, \bm d_B \in \N_D^M}
\binom{D}{\bm d_A}
\binom{D}{\bm d_B}
\E_{\bm t}
\sqbra{
\bm t^{\bm d_A + \bm d_A }
}
\E_{S_N}
\sqbra{
\paren{
\bm p_{\bm d_A}'
\cdot
\bm p_{\bm d_B}'
}
}
.
\label{eqn:RN_in_supp}
\end{align}
The inner product in the expectation over $S_N$ is equivalent to $\paren{\bm P \bm P^\top}_{\bm d_A \bm d_B}$ where $\bm P$ is the control point matrix defined in main body.
So we need to massage the expectation over $\bm t \sim U(\Delta^{M-1})$.
To do it, the following proposition is useful.
\begin{prop}\label{thm:integ_formula}
Suppose $\sum_{i=1}^M q_i = Q \in \N$ and $r \in \R$, then the following integral formula is satisfied.
\begin{align}
I_{\bm q} (r)
&:=
\paren{
\prod_{i=1}^M
\int_0^\infty
 dt_i t_i^{q_i}
}
\delta \paren{
r - \sum_{i=1}^M t_i
}
= 
r^{Q+M-1} \frac{Q!}{(Q+M-1)! } 
\binom{Q}{\bm q}^{-1}
,
\end{align}
where $\delta (r)$ is the Dirac's delta.
\end{prop}
\begin{proof}
One of the easiest proofs is done by making the use of the Laplace transform $\mathcal{L}$ and its inverse $\mathcal{L}^{-1}$, i.e. the nature of $I_{\bm q}(r) = \mathcal{L}^{-1} \sqbra{\mathcal{L} \sqbra{I_{\bm q}}} (r)$.
First, the Laplace transform of $I_{\bm q}$ is
\begin{align}
\mathcal{L} [ I_{\bm q} ] (z)
=
\int_0^\infty dr \ e^{-zr} I_{\bm q} (r)
&=
\Big(
\prod_{i=1}^M
\int_0^\infty
 dt_i t_i^{q_i}
 e^{-zt_i}
\Big)
\notag \\
&=
\frac{Q!}{z^{Q+M}}
\binom{Q}{\bm q}^{-1}.
\end{align}
The inverse Laplace transform of a function $f(z)$ is defined by picking up the residue, the coefficient of $1/z$, of $e^{-zr} f(z)$.
One can calculate the residue of $e^{-zr} \mathcal{L} [ I_{\bm q} ] (z)$ by expanding $e^{-zr}$ with respect to $z$.
It provides $z^{Q+M-1}r^{Q+M-1}/{(Q+M-1)!}$, so we get
\[
\mathcal{L}^{-1} [\mathcal{L} [ I_{\bm q} ]] (r)
=
r^{Q+M-1} \frac{Q!}{(Q+M-1)! } 
\binom{Q}{\bm q}^{-1}.
\]
\end{proof}
Now, let us get back to the proof of the theorem.
The expectation value $\E_{\bm t}[ f(\bm t) ]$ can be represented by the following integral
\[
\E_{\bm t} [f(\bm t)]
=
(M-1)!
\Big(
\prod_{i=1}^M
\int_0^\infty
 dt_i
\Big)
\delta \Big(
1 - \sum_{i=1}^M t_i
\Big)
f(\bm t)
\]
The multiplication of $(M-1)!$ is necessary because we need $\E_{\bm t}[1] = 1$.
This is easily checked by using the above proposition by taking $q_i =0$ for all $i$ and $r=1$.
We can also calculate the value $\E_{\bm t}[ \bm t^{\bm d_A + \bm d_B} ]$ by using the proposition with $r=1$ as
\[
\E_{\bm t}[ \bm t^{\bm d_A + \bm d_B} ]
=
\frac{(2D)!(M-1)!}{(2D + M -1)!}
\binom{2D}{\bm d_A + \bm d_B}^{-1}.
\]
By substituting it to \cref{eqn:RN_in_supp}, we get what we want
\begin{align}
R_N
=
\sum_{\bm d_A, \bm d_B \in \N_D^M}
\frac{(2D)!(M-1)!}{(2D + M -1)!}
\binom{D}{\bm d_A}
\binom{D}{\bm d_B}
\binom{2D}{\bm d_A + \bm d_B}^{-1}
\E_{S_N}
\Bigg[
(
\bm P \bm P^\top
)_{\bm d_A \bm d_B}
\Bigg].
\end{align}

\section{Derivation of \texorpdfstring{$\sum_{i, j = 1}^N (A^{-1} \circ A)_{ij} = N$}{sum of Hadamard products}}\label{sec:derivation-of-sum}
\begin{prop}
Let $A = (a_{ij})$ be a regular matrix of size $N$.
Then
\[
    \sum_{i, j = 1}^N (A^{-1} \circ A)_{ij} = N,
\]
where $A \circ B$ is Hadamard product of $A$ and $B$.
\end{prop}
\begin{proof}
Let $\tilde{A}$ be a cofactor matrix and $C_{ij}$ be a $(i, j)$ minor of $A$.
\begin{equation*}
\begin{split}
\sum_{i, j = 1}^N (A^{-1} \circ A)_{ij}
    &= \sum_{i, j = 1}^N \left(\frac{1}{\det(A)}\tilde{A} \circ A\right)_{ij} \\
    &= \frac{1}{\det(A)}\sum_{i, j = 1}^N (\tilde{A}\circ A)_{ij} \\
    &= \frac{1}{\det(A)}\sum_{i, j = 1}^N a_{ij}C_{ij} \\
    &= \frac{1}{\det(A)}\sum_{i = 1}^N \sum_{j = 1}^N a_{ij}C_{ij}\\
    &= \frac{1}{\det(A)}\sum_{i = 1}^N \det(A) \\
    &= N
\end{split}
\end{equation*}
\end{proof}

\section{Proof of \texorpdfstring{\cref{thm:sigma-is-non-singular}}{Theorem 3}}\label{sec:proof-of-regularity}
Put
\begin{align}
\ZZ_{d_A d_B}
:=
\frac{(2D)!(M - 1)!}{(2D + M - 1)!}
 \begin{pmatrix}
  D \\
  \bm{d}_A
 \end{pmatrix}
 \begin{pmatrix}
  D \\
  {\bm d}_B
 \end{pmatrix}
 \begin{pmatrix}
  2D \\
  {\bm d}_A + {\bm d}_B
 \end{pmatrix}^{-1}.
\end{align}
We prove $\ZZ = (\ZZ_{AB})$ is a regular matrix by considering the representation of appropriate spaces.
We start to give some review on bilinear form.
\begin{defn}
Let $V$ be a real vector space. A map $B\colon V\times V \rightarrow \R$ is said to be bininear form if for any $\bm{u}, \bm{v}, \bm{w} \in V$ and $\lambda \in \R$,
\begin{eqnarray*}
\bullet B(\bm{u}+ \bm{v}, \bm{w}) = B(\bm{u}+\bm{w}) + B(\bm{v}, \bm{w})\ \text{and}\  B(\lambda \bm{u}, \bm{v}) = \lambda B(\bm{u}, \bm{v}) \\
\bullet B(\bm{u}, \bm{v}+\bm{w})  = B(\bm{u}+\bm{v}) + B(\bm{u}, \bm{w})\ \text{and}\  B( \bm{u},\lambda \bm{v}) = \lambda B(\bm{u}, \bm{v})
\end{eqnarray*}
holds. Moreover, a bilinear form $B$ is called symmetric if $B(\bm{u}, \bm{v})=B(\bm{v}, \bm{u})$ holds for any $\bm{u}, \bm{v} \in V$

\end{defn}

\begin{thm}
Let $V$ be an $n$-dimensional real vector space and $\set{\bm{e}_1, \dots, \bm{e}_n}$ be a basis of $V$. Let  $B\colon V\times V \rightarrow \R$ be a bilinear form. Put $A = (B(\bm{e}_i, \bm{e_j}))$. Then the map

\[
\begin{array}{ccc}
B':V \times V  & \longrightarrow & \R\\
\ \ \ \ \  \rotatebox{90}{$\in$}    &                 & \rotatebox{90}{$\in$} \\
\ \ \ \ \ (\bm{u}, \bm{v})                   & \longmapsto     & \bm{x}^{\top} A\bm{y}
\end{array},
\]
is a bilinear form which is equal to $B$, where $\bm{u}=\sum_{i=1}^n x_i e_i$ and $\bm{v}=\sum_{i=1}^n y_i e_i$.
\end{thm}
We call the matrix $A$ the representation matrix of $B$ with respect to a basis  $\set{\bm{e}_1, \dots, \bm{e}_n}$ .  By the construction of $A$, we have the following proposition.
\begin{prop}
Let $V$ be an $n$-dimensional real vector space and   $B\colon V\times V \rightarrow \R$ be a bilinear form. Then the representation matrix of $B$ with respect to a basis of $V$ is a symmetric matrix if and only if $B$ is a symmetric bilinear form.
\end{prop}
\begin{thm}\label{thm:Sigma-is-regular}
$\ZZ$ is a regular matrix.
\end{thm}
\begin{proof}
Let $V_{M, D}$ is the vector space of the homogenous polynomials of degree $D$ with $M$ variables $\bm{x}$.
We define the symmetric bilinear form  by
\[
\begin{array}{ccc}
L:V_{M, D} \times V_{M, D} & \longrightarrow & \R\\
\ \ \ \  \rotatebox{90}{$\in$}    &                 & \rotatebox{90}{$\in$} \\
\ \ \ \  (P, Q)                   & \longmapsto     & \int_{\Delta} P(\bm{x}) Q(\bm{x}) d\bm{x}
\end{array},
\]
where $\Delta = \Set{\bm{x} = (x_1, \dots, x_M) \in \R^M | x_i >0,  x_1 + \dots + x_M = 1}$.

$L$ is clearly symmetric and bilinear.
Let
\[
    Z = \Set{\frac{\bm{x}^{\bm{d_A}}}{\bm{d_A}!} | \bm{d_A} = (d_A^{(1)}, \dots, d_A^{(M)}) \in \N^M, d_A^{(1)} + \dots + d_A^{(M)} = D}
\]
be a set of polynomials of degree $D$ with $M$ variables.
Then $Z$ is a basis of $V_{M, D}$.
We claim that the representation matrix $G$ of the symmetric bilinear form $L$ with respect to $Z$ is equal to $\ZZ$ up to scalar multiplication.

\begin{claim}\label{thm:rep}
Let $\ZZ$, $Z$, $L$ as above.
Then the representation matrix of $L$ with respect to $Z$ is equal to $r\cdot \ZZ$ for some $r \neq 0 \in \R$.
\end{claim}
\begin{proof}[Proof of \cref{thm:rep}]
We have
\begin{equation*}
\begin{split}
L\paren{\frac{\bm{x}^{\bm{d_A}}}{\bm{d_A}!}, \frac{\bm{x}^{\bm{d_B}}}{\bm{d_B}!}}
  & = \frac{1}{\bm{d_A}!\bm{d_B}!}\int_{\Delta} \bm{x}^{\bm{d_A} + \bm{d_B}}d\bm{x}\\
  & = \frac{1}{\bm{d_A}!\bm{d_B}!} \frac{\Gamma(d_A^{(1)} + d_B^{(1)}) \dots \Gamma(d_A^{(M)} + d_B^{(M)}) }{\Gamma(2D + M)} \\
  & = \frac{(\bm{d_A} + \bm{d_B})!}{\bm{d_A}!\bm{d_B}!}\frac{1}{(2D + M - 1)!}\\
  & = \frac{1}{(D!)^2(M - 1)!} \frac{D!}{\bm{d_A}!} \frac{D!}{\bm{d_B}!} \paren{\frac{2D!}{(\bm{d_A + \bm{d_B}})!}}^{-1} \frac{(2D)!(M - 1)!}{(2D + M - 1)!} \\
  & = \frac{1}{(D!)^2(M - 1)!} \ZZ_{\bm{d_A}\bm{d_B}}.
\end{split}
\end{equation*}
Hence, the claim holds.
\end{proof}
Since $G$ is a symmetric matrix, $G$ is diagonalizable.
Hence, for some basis $Z'$, the representation matrix $H$ of the symmetric bilinear form $L$ with respect to $Z'$ is a diagonal matrix.
In this case, the diagonal components of $H$ is equal to the eigenvalues of $G$ and is equal to $L(P, P)$ for some $P \in Z'$.
Hence, for proving our theorem, it is enough to show that $L(P, P)$ is positive.
To see the positivity of $L(P, P)$, we claim the following.

\begin{claim}\label{thm:hom}
Let $P(\bm{x})$ be a homogenous polynomial of degree $D$ with $M$ variables.
Put $\Delta = \Set{\bm{x} = (x_1, \dots, x_M) \in \R^M | x_i > 0,\ x_1 + \dots + x_M = 1}$.
Assume $\int_{\Delta} P(\bm{x})^2 d \bm{x} = 0$.
Then $P(\bm{x}) = 0$.
\end{claim}

\begin{proof}[Proof of \cref{thm:hom}]
Assume $P(\bm{x}) \neq 0$.
Since $\int_{\Delta} P(\bm{x})^2 d \bm{x} = 0$, we see $P(\bm{x}) = 0$ for any $\bm{x} \in \Delta$.
Since $P(\bm{x}) \neq 0$, we have a point $\bm{a} = (a_1, \dots, a_M) \in \R^M_+$ such that $P(\bm{a}) \neq 0$.
Put $\lambda = \sum a_i$. Then, since $P(\bm{x})$ is a homogeneous polynomial of degree $D$, we have
\[
    P(\bm{a}) = \lambda^D P(\frac{1}{\lambda} \bm{a}) \neq 0.
\]
The non equality holds because $\sum 1 / \lambda \cdot a_i \in \Delta$.
\end{proof}

Since $P$ is a nonzero polynomial, we have
\[
    (P, P) = \int_{\Delta} P(\bm{x})^2 d \bm{x} > 0
\]
by \cref{thm:hom}.
Therefore, any eigenvalue of $G$ is positive.
This implies that $\ZZ$ is a regular matrix.
\end{proof}
\section{Complete derivation of the risk asymptotics in inductive skeleton fitting}\label{sec:risk-derivation}
Let us begin with how control points on $(m-1)$-subsimplices from minimizing the OLS error
\begin{align}
\frac{1}{N^{(m)}}
\sum_{n=1}^{N^{(m)}}
\norm{
\bm x_n^{(m)}
-
\hat{\bm b} (\bm t_n^{(m)})
}^2
.
\end{align}
In this paper, we assume that sample $\bm x_n$ always represented by sum of a certain B\'esier simplex and noise:
\begin{align}
\bm x_n^{(m)}
=
\bm b (\bm t_n^{(m)})
+
\bm \epsilon_n^{(m)},
\quad
\epsilon_n^{(m)} \sim \mathcal{N}(\bm 0, \sigma^2 \bm I_L)
\end{align}
\paragraph{all-at-once fitting}
All-at-once fitting in our paper can be regarded $m=M$ case.
In this case, we omit the superscript $(m=M)$ and consider minimization of
\begin{align}
\frac{1}{N}
\sum_{n=1}^{N}
\norm{
\bm x_n
-
\hat{\bm b} (\bm t_n)
}^2
&=
\frac{1}{N}
\sum_{n=1}^{N}
\Big| \Big|
{\bm b} (\bm t_n)
-
\hat{\bm b} (\bm t_n)
+
\bm \epsilon_n
\Big| \Big|^2
\notag \\
&=
\frac{1}{N}
\sum_{n=1}^{N}
\Big| \Big|
\sum_{\bm d \in \N_D^M}
\binom{D}{\bm d} \bm t^{\bm d}_n
\bm p_{\bm d}'
+
\bm \epsilon_n
\Big| \Big|^2
\notag \\
&=
\frac{1}{N}
\norm{
\bm Z \bm P
+
\bm Y
}_\mathrm{F}^2,
\end{align}
where the $\norm{\cdot}_\mathrm{F}$ means Frobenius norm of the matrix, and the optimum is determined as
$\bm P_\mathrm{OLS} = - (\bm Z^\top \bm Z)^{-1} \bm Z^\top \bm Y$ by using usual argument.
\paragraph{inductive skeleton fitting}
On the other hand, the above argument should be modified when we consider the inductive skeleton fitting.
First of all, we can reduce the problem to minimization of
\begin{align}
\frac{1}{N^{(m)}}
\sum_{n=1}^{N^{(m)}}
\Big| \Big|
\sum_{\bm d \in \N_D^M}
\binom{D}{\bm d} (\bm t_n^{(m)})^{\bm d}
\bm p_{\bm d}'
+
\bm \epsilon_n^{(m)}
\Big| \Big|^2.
\label{eqn:ols_inductive_supp}
\end{align}
Note that $\bm t_n^{(m)}$ is on $(m-1)$-subsimplices.
It means that the vector
\begin{align}
\bm t_n^{(m)}
=
[
(\bm t_n)_1,
(\bm t_n)_2,
\dots ,
(\bm t_n)_M
]
\end{align}
has $(M-m)$ zero components.
For example, if it is on a subsimplex labelled by
\begin{align}
\bm I = [1, 1, \dots, 1, 0, 0, \dots, 0],
\end{align}
with $m$ ones and $(M-m)$ zeros, then $\bm t_n^{(m)}$ takes
\begin{align}
\bm t_n^{(m)}
=
[
\text{non-zero},
\text{non-zero},
\dots ,
\text{non-zero},
0,
0,
\dots,
0
].
\end{align}
As a result, there is no contribution from $\bm p_{\bm d}'$ which is not on the simplex because such point is labelled by
\begin{align}
\bm d
=
[
d_1,
d_2,
\dots,
d_m,
\text{non-zero},
\text{non-zero},
\dots,
\text{non-zero}
]
\end{align}
because
\begin{align}
(t_n^{(m)})^{\bm d}
=
(\text{non-zero})_1^{d_1}
(\text{non-zero})_2^{d_2}
\dots
(\text{non-zero})_m^{d_m}
(0)^{\text{non-zero}}
(0)^{\text{non-zero}}
\dots
(0)^{\text{non-zero}}
\end{align}
vanishes inside the summation of \cref{eqn:ols_inductive_supp}.
In summary, we should restrict range of the summation in \cref{eqn:ols_inductive_supp} as
\begin{align}
&
\frac{1}{N^{(m)}}
\sum_{n=1}^{N^{(m)}}
\Big| \Big|
\sum_{\bm d \in (\N_D^M)^{(m)}}
\binom{D}{\bm d} (\bm t_n^{(m)})^{\bm d}
\bm p_{\bm d}'
+
\sum_{k < m}
\sum_{\bm d \in (\N_D^M)^{(k)}}
\binom{D}{\bm d} (\bm t_n^{(m)})^{\bm d}
\bm p_{\bm d}'
+
\bm \epsilon_n^{(m)}
\Big| \Big|^2
\notag \\
&=
\frac{1}{N^{(m)}}
\norm{
\bm Z^{(m)} \bm P^{(m)}
+
\sum_{k<m}
\bm Z^{(m)[k]} \bm P^{(k)}
+
\bm Y^{(m)}
}^2_\mathrm{F}.
\label{eqn:ols_ind_supp2}
\end{align}
In the inductive skeleton fitting, we regard all $\bm P^{(k<m)}$ are already determined and fixed in the optimization process of $\bm P^{(m)}$, so the last 2 terms in \cref{eqn:ols_ind_supp2} play a role of $\bm Y$ in all-at-once fitting, and the optimal $\bm P^{(m)}$ is determined as
\begin{align}
{\bm P}_\mathrm{OLS}^{(m)}
= 
-
[({\bm Z}^{(m)})^\top {\bm Z}^{(m)}]^{-1}
({\bm Z}^{(m)})^\top
\Big(
{\bm Y}^{(m)}
+
\sum_{k<m}
{\bm Z}^{(m) [k]}
{\bm P}_\mathrm{OLS}^{(k)}
\Big).
\label{eqn:rec_supp}
\end{align}

To derive the asymptotics of the risk, we have to consider $\E_\text{samples} [ \bm P \bm P^\top ]$.
In the all-at-once fitting, it is simple as explained in the main body of the paper.
The key idea is calculating the expectation over the noise $\bm Y$ first and making use of the law of large numbers.
On the inductive skeleton fitting, we take same strategy.
Let us call the control point matrix $\bm P$ determined by the inductive skeleton fitting as $\bm P_\text{ISK}$.
It decomposes to ``direct sum'' along row of the matrix as
\begin{align}
\bm P_\text{ISK}
=
\begin{pmatrix}
{\bm P}^{(1)}_\mathrm{OLS}
\\
{\bm P}^{(2) }_\mathrm{OLS}
\\
\vdots
\\
{\bm P}^{(M)}_\mathrm{OLS}
\end{pmatrix}
.
\end{align}
If the degree of the B\'ezier simplex is less than $m$, there may be no control point on $(m+k)$-subsimplices.
In such case, we regard $\bm P_\mathrm{OLS}^{(m+k)} = \bm 0$.
By using the decomposition, we can get the ``direct sum'' of the matrix $\bm P_\text{ISK} \bm P_\text{ISK}^\top$ as follows.
\begin{align}
\bm P_\text{ISK} \bm P_\text{ISK}^\top
=
\begin{pmatrix}
{\bm P}^{(1)}_\mathrm{OLS}
({\bm P}^{(1)}_\mathrm{OLS})^\top
&
{\bm P}^{(1)}_\mathrm{OLS}
({\bm P}^{(2)}_\mathrm{OLS})^\top
&
\cdots
&
{\bm P}^{(1)}_\mathrm{OLS}
({\bm P}^{(M)}_\mathrm{OLS})^\top
\\
{\bm P}^{(2)}_\mathrm{OLS}
({\bm P}^{(1)}_\mathrm{OLS})^\top
&
{\bm P}^{(2)}_\mathrm{OLS}
({\bm P}^{(2)}_\mathrm{OLS})^\top
&
\cdots
&
{\bm P}^{(2)}_\mathrm{OLS}
({\bm P}^{(M)}_\mathrm{OLS})^\top
\\
\vdots & \vdots & \ & \vdots
\\
{\bm P}^{(M)}_\mathrm{OLS}
({\bm P}^{(1)}_\mathrm{OLS})^\top
&
{\bm P}^{(M)}_\mathrm{OLS}
({\bm P}^{(2)}_\mathrm{OLS})^\top
&
\cdots
&
{\bm P}^{(M)}_\mathrm{OLS}
({\bm P}^{(M)}_\mathrm{OLS})^\top
\end{pmatrix}
.
\end{align}
Therefore, to get the asymptotic form of the risk, it is sufficient to consider $\E_\text{samples}[{\bm P}^{(i)}_\mathrm{OLS}
({\bm P}^{(j)}_\mathrm{OLS})^\top]$ for $i, j = 1, 2, \dots, M$.
To get a grasp of how it calculated, let us take $(i,j) = (1,1)$ and $(i, j) =  (2, 1)$ here.
First, $(i,j) = (1,1)$ is the simplest case because there is no lower-dimensional subsimplex.
So we do not need 2nd term in \cref{eqn:rec_supp} and get
\begin{align}
&\E_\text{samples}
\Big[
\bm P_\mathrm{OLS}^{(1)}
(\bm P_\mathrm{OLS}^{(1)})^\top
\Big]
\notag \\
&=
\E_\text{samples}
\Bigg[
\Big[
-
[({\bm Z}^{(1)})^\top {\bm Z}^{(1)}]^{-1}
({\bm Z}^{(1)})^\top
{\bm Y}^{(1)}
\Big]
\Big[
-
({\bm Y}^{(1)})^\top
{\bm Z}^{(1)}
[({\bm Z}^{(1)})^\top {\bm Z}^{(1)}]^{-1}
\Big]
\Bigg]
\notag \\
&=
\E_{\text{samples}}
\Bigg[
[({\bm Z}^{(1)})^\top {\bm Z}^{(1)}]^{-1}
({\bm Z}^{(1)})^\top
\E_{\bm Y^{(1)}}
\Big[
{\bm Y}^{(1)}
({\bm Y}^{(1)})^\top
\Big]
{\bm Z}^{(1)}
[({\bm Z}^{(1)})^\top {\bm Z}^{(1)}]^{-1}
\Bigg]
\notag \\
&=
\E_{\text{samples}}
\Bigg[
[({\bm Z}^{(1)})^\top {\bm Z}^{(1)}]^{-1}
({\bm Z}^{(1)})^\top
\Big[
\sigma^2 L 1_{N^{(1)}}
\Big]
{\bm Z}^{(1)}
[({\bm Z}^{(1)})^\top {\bm Z}^{(1)}]^{-1}
\Bigg]
\notag \\
&=
\sigma^2 L
\E_{\text{samples}}
\Bigg[
[({\bm Z}^{(1)})^\top {\bm Z}^{(1)}]^{-1}
\Bigg].
\label{eqn:P1P1}
\end{align}
In fact, this simplification occurs in calculation with all-at-once fitting.
So $(i, j) = (1,1)$ case is in same situation of the calculation of the asymptotic risk of the all-at-once fitting.
How about $(i, j) = (2, 1)$?
This is also simple because we can reduce the calculation to lower-dimensional object by using \cref{eqn:rec_supp}:
\begin{align}
&\E_\text{samples}
\Big[
\bm P_\mathrm{OLS}^{(2)}
(\bm P_\mathrm{OLS}^{(1)})^\top
\Big]
\notag \\
&=
\E_\text{samples}
\Bigg[
\Big[
-
[({\bm Z}^{(2)})^\top {\bm Z}^{(2)}]^{-1}
({\bm Z}^{(2)})^\top
\Big(
{\bm Y}^{(2)}
+
\bm Z^{(2)[1]}
\bm P_\mathrm{OLS}^{(1)}
\Big)
\Big]
(\bm P_\mathrm{OLS}^{(1)})^\top
\Bigg]
\notag \\
&=
\E_\text{samples}
\Bigg[
\Big[
-
[({\bm Z}^{(2)})^\top {\bm Z}^{(2)}]^{-1}
({\bm Z}^{(2)})^\top
\Big(
\E_{\bm Y^{(2)}}
[
{\bm Y}^{(2)}
]
+
\bm Z^{(2)[1]}
\bm P_\mathrm{OLS}^{(1)}
\Big)
\Big]
(\bm P_\mathrm{OLS}^{(1)})^\top
\Bigg]
\notag \\
&=
\E_\text{samples}
\Bigg[
\Big[
-
[({\bm Z}^{(2)})^\top {\bm Z}^{(2)}]^{-1}
({\bm Z}^{(2)})^\top
\Big(
[
0
]
+
\bm Z^{(2)[1]}
\bm P_\mathrm{OLS}^{(1)}
\Big)
\Big]
(\bm P_\mathrm{OLS}^{(1)})^\top
\Bigg]
\notag \\
&=
- \sigma^2 L
\E_\text{samples}
\Bigg[
[({\bm Z}^{(2)})^\top {\bm Z}^{(2)}]^{-1}
[({\bm Z}^{(2)})^\top \bm Z^{(2)[1]}]
[({\bm Z}^{(1)})^\top {\bm Z}^{(1)}]^{-1}
\Bigg].
\label{eqn:P2P1}
\end{align}
With generic $(i, j)$, we can perform same procedure.

Now, let us make next step.
To calculate the asymptotic form, it is important to notice that the following explicit forms are averages on sampling from $\bm t_n^{(m)} \sim U(\cup_{|\bm I|=m} \Delta^{\bm I})$.
\begin{align}
&
\frac{1}{N^{(m)}}
\Big(
(\bm Z^{(m)})^\top
\bm Z^{(m)[k]}
\Big)_{\bm d_m \bm d_k}
=
\binom{D}{\bm d_m}
\binom{D}{\bm d_k}
\sum_{n=1}^{N^{(m)}}
\frac{1}{N^{(m)}}
(\bm t_n^{(m)})^{\bm d_m + \bm d_k}
\label{eqn:Zmk}
\\
&
\frac{1}{N^{(m)}}
\Big(
(\bm Z^{(m)})^\top
\bm Z^{(m)}
\Big)_{\bm d_A \bm d_B}
=
\binom{D}{\bm d_A}
\binom{D}{\bm d_B}
\sum_{n=1}^{N^{(m)}}
\frac{1}{N^{(m)}}
(\bm t_n^{(m)})^{\bm d_A + \bm d_B}
\label{eqn:Zmm}
\end{align}
By applying the law of large numbers, they converge to exact values of expectation:
\begin{align}
&
\cref{eqn:Zmk}
\to
\binom{D}{\bm d_m}
\binom{D}{\bm d_k}
\E_{\bm t^{(m)} \sim U(\cup_{|\bm I|=m} \Delta^{\bm I})} [
(\bm t^{(m)})^{\bm d_m + \bm d_k}
]
\label{eqn:Lmk}
\\
&
\cref{eqn:Zmm}
\to
\binom{D}{\bm d_A}
\binom{D}{\bm d_B}
\E_{\bm t^{(m)} \sim U(\cup_{|\bm I|=m} \Delta^{\bm I})} [
(\bm t^{(m)})^{\bm d_A + \bm d_B}
]
\label{eqn:Lmm}
\end{align}
To compute expectation values over $U(\cup_{|\bm I|=m} \Delta^{\bm I})$, the next proposition is useful.
\begin{prop}
\begin{align}
\E_{\bm t \sim U(\cup_{|\bm I|=m} \Delta^{\bm I})} [
\bm t^{\bm d^{(m)} + \bm d^{(k)}}
]
=
\frac{(m-1)!}{\ _M C_m}
\sum_{|\bm I| = m}
\frac{
1_{{\bm I}=(\bm d ^{(m)} +  \bm d ^{(k)} )_{01} }
\prod_{I_i = 1}( \bm d ^{(m)} +  \bm d ^{(k)}  )_i !}{
[\sum_{I_i = 1} ( \bm d ^{(m)} +  \bm d ^{(k)}  )_i
+ m - 1]!
},
\end{align}
where $\bm I = [I_1, I_2, \dots, I_M]$ represents vector with binary components, $I_i \in \set{0, 1}$ and $\card{\bm I} = \sum_{i=1}^M I_i$.
\end{prop}
\begin{proof}
The expectation value with the uniform distribution on $(m-1)$-subsimplices should be expressed by integral
\begin{align}
\E_{\bm t \sim U(\cup_{|\bm I|=m} \Delta^{\bm I})} [ f(\bm t) ]
=
(\text{const.})
\int_{0}^1 \Big(
\prod_{i=1}^M dt_i
\Big)
\ f(\bm t)
\sum_{|\bm I| = m}
\delta \Big(
1 - \sum_{I_i =1} t_i
\Big)
\prod_{I_i =0}
\delta (t_i)
\label{eqn:exp_sub}
\end{align}
because it can be regarded integral over $\Delta^{M}$ restricted to the set of subsimplices.
The normalization constant can be calculated by the integral $f(\bm t) = 1$ inserted.
In this situation, there is no special direction on $\bm t$, and all contributions from the sum $\sum_{|\bm I|=m}$ should be identical.
Therefore, we do not need to calculate all of them but just select one configuration $\bm I$ easy to calculate and multiply the number of combination giving $|\bm I|=m$, i.e. $\ _M C_m$.
Let us take
\begin{align}
\bm I = [\underbrace{1,1, \dots, 1}_{m}, 0, 0, \dots, 0]
\end{align}
then the integral reduces to
\begin{align}
\int_{0}^1 \Big(
\prod_{i=1}^m dt_i
\Big)
\delta \Big(
1 - \sum_{i=1}^m t_i
\Big)
=
\frac{1}{(m-1)!}
\end{align}
where we use the integral formula derived in \ref{thm:integ_formula}.
It determines the constant in \cref{eqn:exp_sub},
\begin{align}
(\text{const.})
=
\frac{(m-1)!}{\ _M C_m}
.
\end{align}
Next task is performing the integral $f(\bm t) = \bm t^{\bm d^{(m)} + \bm d^{(k)}}$ inserted.
It is also not so difficult.
To get nonzero value from $\bm I$-th contribution, $\bm d^{(m)} + \bm d^{(k)}$ should be on subsimplex labelled by $\bm I$.
If not so, the integrand includes
\begin{align}
\int dt \
t^\text{non-zero}
\delta(t)
=
0^\text{non-zero}
=0.
\end{align}
To represent this condition, let us introduce the notation
\begin{align}
(\bm d)_{01}
=
[d^\mathrm{bin}_1, d^\mathrm{bin}_2, \dots, d^\mathrm{bin}_M],
\quad
d^\mathrm{bin}_i
=
\begin{cases}
0 & \text{if $d_i = 0$}\\
1 & \text{otherwise}
\end{cases}
\end{align}
By using this notation, the nonzero contribution condition is represented by the insertion
\begin{align}
1_{\bm I = (\bm d^{(m)} + \bm d^{(k)})_{01} }.
\end{align}
If index $\bm I$ in the summation $\sum_{|\bm I| = m}$ enjoys this condition, the remaining part is integral on an $(m-1)$-subsimplex labelled by $\bm I$, and the contribution is
\begin{align}
\frac{\prod_{I_i=1} (\bm d^{(m)} + \bm d^{(k)})_i ! }{
[
\sum_{I_i=1}(\bm d^{(m)} + \bm d^{(k)})_i + m -1 ]!
}
\end{align}
which is derived in the same way of the integral formula shown in \ref{thm:integ_formula}.
\end{proof}
Applying this proposition to \cref{eqn:Lmk,eqn:Lmm}, we get the asymptotic formula of $
\E_\text{samples}[
\bm P_\text{ISK} \bm P_\text{ISK}^\top]$, or equivalently, $\E_\text{samples}[\bm P_\mathrm{OLS}^{(i)} (\bm P_\mathrm{OLS}^{(j)})^\top]$.
For example,
\begin{align}
\E_\text{samples}[\bm P_\mathrm{OLS}^{(1)} (\bm P_\mathrm{OLS}^{(1)})^\top]
&=
\cref{eqn:P1P1}
=
\frac{\sigma^2 L}{N^{(1)}} \E_\text{samples}
\Big[
[\frac{1}{N^{(1)}} (\bm Z^{(1)})^\top \bm Z^{(1)} ]^{-1}
\Big]
\notag \\
& \to
\frac{\sigma^2 L}{N^{(1)}} \E_\text{samples}
\Big[
\bm \Lambda_{(1)}
\Big]
=
\frac{\sigma^2 L}{N^{(1)}} \bm \Lambda_{(1)},
\end{align}
and
\begin{align}
&\E_\text{samples}[\bm P_\mathrm{OLS}^{(2)} (\bm P_\mathrm{OLS}^{(1)})^\top]
\notag \\
&=
\cref{eqn:P2P1}
=
- \frac{\sigma^2 L}{N^{(1)}}
\E_\text{samples}
\Bigg[
[\frac{1}{N^{(2)}} ({\bm Z}^{(2)})^\top {\bm Z}^{(2)}]^{-1}
[\frac{1}{N^{(2)}}({\bm Z}^{(2)})^\top \bm Z^{(2)[1]}]
[\frac{1}{N^{(1)}}({\bm Z}^{(1)})^\top {\bm Z}^{(1)}]^{-1}
\Bigg]
\notag \\
& \to
- \frac{\sigma^2 L}{N^{(1)}}
\E_\text{samples}
\Bigg[
\bm \Lambda_{(2)}
\bm \Lambda^{(2)[1]}
\bm \Lambda_{(1)}
\Bigg]
=
- \frac{\sigma^2 L}{N^{(1)}}
\bm \Lambda_{(2)}
\bm \Lambda^{(2)[1]}
\bm \Lambda_{(1)},
\end{align}
where $\bm \Lambda$s are defined in the main body of the paper.
As one can see, to complete our argument, it is sufficient to show the next proposition:
\begin{prop}
\begin{align}
&\E_\text{all noises}
\Big[
\bm P^{(i)}_\mathrm{OLS}
(\bm P^{(j)}_\mathrm{OLS})^\top
\Big]
\notag \\
&=
\sigma^2 L
\sum_{
\substack{
m \leq i
\\
m \leq j
}}
\sum_{
\substack{
m\leq k_1<\cdots<k_{\heartsuit}< i
\\
m \leq l_1<\cdots<l_{\spadesuit}<j
}
}
\frac{(-1)^{\heartsuit + \spadesuit}}{N^{(m)}}
\hat{\bm \Lambda}_{(i)}
\hat{\bm \Lambda}^{(i)[k_{\heartsuit}]}
\hat{\bm \Lambda}_{(k_{\heartsuit} )}
\cdots
\hat{\bm \Lambda}^{(k_1)[m]}
\hat{\bm \Lambda}_{(m)}
\hat{\bm \Lambda}^{[m](l_1)}
\cdots
\hat{\bm \Lambda}_{(l_{\spadesuit} )}
\hat{\bm \Lambda}^{[l_{\spadesuit}](j)}
\hat{\bm \Lambda}_{(j)},
\label{eqn:31}
\end{align}
where
\begin{align}
&\hat{\bm \Lambda}^{(m)[k]}
=
\frac{1}{N^{(m)}}
({\bm Z}^{(m)})^\top
{\bm Z}^{(m)[k]}
\\
&
\hat{\bm \Lambda}^{[k](m)}
=
(\hat{\bm \Lambda}^{(m)[k]})^\top
\\
&
\hat{\bm \Lambda}_{(m)}
=
[\frac{1}{N^{(m)}}({\bm Z}^{(m)})^\top
{\bm Z}^{(m)}]^{-1}
\end{align}
\end{prop}
\begin{proof}
We show it by induction.
For $(i, j) = (1, 1)$, the only possible contribution of the summation in \cref{eqn:31} is $m=1$ and the (RHS) reduces to
\begin{align}
\sigma^2 L
\frac{(-1)^{0+0}}{N^{(1)}}
\hat{\bm \Lambda}_{(1)}
\hat{\bm \Lambda}^{(1)[1]}
\hat{\bm \Lambda}_{(1)}
\hat{\bm \Lambda}^{[1](1)}
\hat{\bm \Lambda}_{(1)}
=
\frac{\sigma^2 L}{N^{(1)}}
\hat{\bm \Lambda}_{(1)},
\end{align}
which is satisfied as we shown already.
Now let us assume \cref{eqn:31} with $(k, j)$ for all $k< i+1$, then, by using the recursive relation \cref{eqn:rec_supp}, we get
\begin{align}
&\E_\text{all noises}
\Big[
\bm P_\mathrm{OLS}^{(i+1)}
(\bm P_\mathrm{OLS}^{(j)})^\top
\Big]
\notag \\
&=
\E_\text{all noises}
\Big[
\Big\{
-
[(\bm Z^{(i+1)})^\top \bm Z^{(i+1)} ]^{-1}
(\bm Z^{(i+1)})^\top
\Big(
\bm Y^{(i+1)}
+
\sum_{k<i+1}
\bm Z^{(i+1)[k]}
\bm P_\mathrm{OLS}^{(k)}
\Big)
\Big\}
(\bm P_\mathrm{OLS}^{(j)})^\top
\Big].
\label{eqn:41}
\end{align}
1st term gives
\begin{align}
&\E_\text{all noises}
\Big[
-
[(\bm Z^{(i+1)})^\top \bm Z^{(i+1)} ]^{-1}
(\bm Z^{(i+1)})^\top
\bm Y^{(i+1)}
(\bm P_\mathrm{OLS}^{(j)})^\top
\Bigg]
\notag \\
&= 
\E_\text{all noises}
\Big[
-
[(\bm Z^{(i+1)})^\top \bm Z^{(i+1)} ]^{-1}
(\bm Z^{(i+1)})^\top
\bm Y^{(i+1)}
\Big\{
-
(\bm Y^{(j)})^\top
\notag \\ & \hspace{180pt}
-
\sum_{m < j}
(\bm P_\mathrm{OLS}^{(m)})^\top
\bm Z^{[m](j)}
\Big\}
\bm Z^{(j)}
[(\bm Z^{(j)})^\top \bm Z^{(j)} ]^{-1}
\Big]
\notag \\
&= 
\E_\text{all noises}
\Big[
\underbrace{
[(\bm Z^{(i+1)})^\top \bm Z^{(i+1)} ]^{-1}
(\bm Z^{(i+1)})^\top
\bm Y^{(i+1)}
(\bm Y^{(j)})^\top
\bm Z^{(j)}
[(\bm Z^{(j)})^\top \bm Z^{(j)} ]^{-1}
}_{(*1)}
\notag \\ &
+
\underbrace{
[(\bm Z^{(i+1)})^\top \bm Z^{(i+1)} ]^{-1}
(\bm Z^{(i+1)})^\top
\bm Y^{(i+1)}
\sum_{m < j}
(\bm P_\mathrm{OLS}^{(m)})^\top
\bm Z^{[m](j)}
\bm Z^{(j)}
[(\bm Z^{(j)})^\top \bm Z^{(j)} ]^{-1}
}_{(*2)}
\Big].
\end{align}
If $i+1 = j$, the first term only contribute because there is no $(\bm Y^{(i+1)})^\top$ in the summation $\sum_{m < j=i+1}$ of the second term, and it gives
\begin{align}
\E (*1)
=
[(\bm Z^{(i+1)})^\top \bm Z^{(i+1)} ]^{-1}
=
\frac{1}{N^{(i+1)}}
\hat{\bm \Lambda}_{(i+1)}
\label{eqn:43}
\end{align}
If $i+1 \neq j$, the first term vanish in the same reason.
In this case, the second term possibly contribute if $i+1 < j$.
By applying the recursive formula \cref{eqn:rec_supp} to $\bm P_\mathrm{OLS}^{(m)}$ repeatedly until $(\bm Y^{(i+1)})^\top$ appears , we get
\begin{align}
\E (*2)
=
\sigma^2 L
\sum_{m < j}
\sum_{m \leq l_1 < \dots < l_\spadesuit <j }
\frac{(-1)^{\spadesuit}}{N^{(i+1)}}
\hat{\bm \Lambda}_{(i+1)}
\hat{\bm \Lambda}^{[i+1](l_1)}
\hat{\bm \Lambda}_{(l_1)}
\dots
\hat{\bm \Lambda}_{(l_\spadesuit)}
\hat{\bm \Lambda}^{[l_\spadesuit](j)}
\hat{\bm \Lambda}_{(j)}.
\label{eqn:44}
\end{align}

On the other hand, the 2nd term in \cref{eqn:41} is
\begin{align}
&
\E_\text{all noises}
\Big[
-
[(\bm Z^{(i+1)})^\top \bm Z^{(i+1)}]^{-1}
(\bm Z ^{(i+1)})^\top
\sum_{k<i+1}
\bm Z^{(i+1)[k]}
\bm P_\mathrm{OLS}^{(k)}
(\bm P_\mathrm{OLS}^{(j)})^\top
\Big]
\notag \\
&=
\sum_{k<i+1}
(-1)
\hat{\bm \Lambda}_{(i+1)}
\hat{\bm \Lambda}^{(i+1)[k]}
\sigma^2 L
\sum_{
\substack{
m \leq k
\\
m \leq j
}}
\sum_{
\substack{
m\leq k_1<\cdots<k_{\heartsuit}<k
\\
m \leq l_1<\cdots<l_{\spadesuit}<j
}
}
\notag \\ & \qquad\qquad\qquad
\frac{(-1)^{\heartsuit + \spadesuit}}{N^{(m)}}
\hat{\bm \Lambda}_{(k)}
\hat{\bm \Lambda}^{(k)[k_{\heartsuit}]}
\hat{\bm \Lambda}_{(k_{\heartsuit} )}
\cdots
\hat{\bm \Lambda}^{(k_1)[m]}
\hat{\bm \Lambda}_{(m)}
\hat{\bm \Lambda}^{[m](l_1)}
\cdots
\hat{\bm \Lambda}_{(l_{\spadesuit} )}
\hat{\bm \Lambda}^{[l_{\spadesuit}](j)}
\hat{\bm \Lambda}_{(j)}
\notag \\
&=
\sigma^2 L
\sum_{
\substack{
m \leq k
\\
m \leq j
}}
\sum_{
\substack{
m\leq k_1<\cdots<k_{\heartsuit}<k<i+1
\\
m \leq l_1<\cdots<l_{\spadesuit}<j
}
}
\hat{\bm \Lambda}_{(i+1)}
\hat{\bm \Lambda}^{(i+1)[k]}
\notag \\ & \qquad\qquad\qquad
\frac{(-1)^{\heartsuit+1 + \spadesuit}}{N^{(m)}}
\hat{\bm \Lambda}_{(k)}
\hat{\bm \Lambda}^{(k)[k_{\heartsuit}]}
\hat{\bm \Lambda}_{(k_{\heartsuit} )}
\cdots
\hat{\bm \Lambda}^{(k_1)[m]}
\hat{\bm \Lambda}_{(m)}
\hat{\bm \Lambda}^{[m](l_1)}
\cdots
\hat{\bm \Lambda}_{(l_{\spadesuit} )}
\hat{\bm \Lambda}^{[l_{\spadesuit}](j)}
\hat{\bm \Lambda}_{(j)}
\notag \\
&= 
\sigma^2 L
\sum_{
\substack{
m \leq k
\\
m \leq j
}}
\sum_{
\substack{
m\leq k_1<\cdots<k_{\heartsuit}<k<i+1
\\
m \leq l_1<\cdots<l_{\spadesuit}<j
}
}
\notag \\ &
\frac{(-1)^{\heartsuit+1 + \spadesuit}}{N^{(m)}}
\hat{\bm \Lambda}_{(i+1)}
\hat{\bm \Lambda}^{(i+1)[k]}
\hat{\bm \Lambda}_{(k)}
\hat{\bm \Lambda}^{(k)[k_{\heartsuit}]}
\hat{\bm \Lambda}_{(k_{\heartsuit} )}
\cdots
\hat{\bm \Lambda}^{(k_1)[m]}
\hat{\bm \Lambda}_{(m)}
\hat{\bm \Lambda}^{[m](l_1)}
\cdots
\hat{\bm \Lambda}_{(l_{\spadesuit} )}
\hat{\bm \Lambda}^{[l_{\spadesuit}](j)}
\hat{\bm \Lambda}_{(j)}
\end{align}
This is close to what we want to show by renaming $k = k_{\heartsuit + 1}$, but it lacks contribution from $m = i+1$:
\begin{align}
\sigma^2 L
\sum_{m \leq j}
\sum_{m \leq l_1 < \dots < l_\spadesuit <j }
\frac{(-1)^{\spadesuit}}{N^{(i+1)}}
\hat{\bm \Lambda}_{(i+1)}
\hat{\bm \Lambda}^{[i+1](l_1)}
\hat{\bm \Lambda}_{(l_1)}
\dots
\hat{\bm \Lambda}_{(l_\spadesuit)}
\hat{\bm \Lambda}^{[l_\spadesuit](j)}
\hat{\bm \Lambda}_{(j)}
.
\end{align}
But is is compensated by \cref{eqn:43,eqn:44}, if $j=i+1$ and $j \neq i+1$ respectively, so we succeed in deriving the equation with $(i+1, j)$.

On $(i, j+1)$ case, we can show it by repeating the above argument in transposed version.
\end{proof}
Now we complete the derivation of the asymptotic form of $\E_\text{samples}[\bm P_\text{ISK} \bm P_\text{ISK}^\top]$.

The following proposition gives the optimal value of the risk of inductive skeleton fitting in the case of $D=2$.
\begin{prop}
Let $N$ be a natural number and $a, b$ be positive real numbers.
Let $f(x) = \frac{a}{x} + \frac{b}{N - x}$.
Assume $a > b$ and $\frac{a - \sqrt{ab}}{a - b} < 1$.
Then $\min \Set{f(x) | x \in \R, 0 < x < N} = \frac{a + b + 2 \sqrt{ab}}{N}$.
\end{prop}

\begin{proof}
At first, we consider the differential of $f(x)$.
We have
\[
    f'(x) = -\frac{(a - b) x^2 - 2aNx + aN^2}{x^2 (N - x)^2}.
\]
$f'(x) = 0$ if and only if $x = \frac{a \pm \sqrt{ab}}{a - b}N$.
By the assumptions, $f(x)$ takes the minimal value at $x = \frac{a - \sqrt{ab}}{a - b}N$.
Hence, we have
\[
    \min \Set{f(x) | x \in \R, 0 < x < N} = f(\frac{a - \sqrt{ab}}{a - b} N) = \frac{a + b + 2 \sqrt{ab}}{N}.
\]
\end{proof}

\section{All experiments}\label{sec:numerical-experiments}
This section provides all experiments we conducted and their detailed settings for completeness and reproducibility.

\subsection{Synthetic instances}
First of all, we present detailed experimental settings and results of synthetic instances described in \cref{sec:synthetic-instances} of the main paper.

\subsubsection{Data generation}
We consider the following data generating process to create synthetic instances.
Given parameters $(L, M, N)$, we first define true control points $\bm p_{\bm d}\in \R^L~(\bm d \in \N^M_D)$ as follows:
\begin{equation}\label{eqn:def_control_points}
    \bm p_{\bm d} := \sum_{j=1}^M \frac{d_j}{D} \bm e_j,
\end{equation}
where $\bm e_j \in \R^L~(j = 1, \dots, M)$ is a unit vector whose $j$-th element is one and the others are zeros.
The B\'ezier simplex defined by \cref{eqn:def_control_points} is a unit $(M-1)$-simplex on $\R^L$ and their $M$ vertices are $\bm e_j~(j = 1, \dots, M)$.

Next, we randomly generated $N$ training points $\set{(\bm t_n, \bm x_n)}_{n = 1}^N$
for the all-at-once fitting and the inductive skeleton fitting respectively.
For the all-at-once fitting, we generated parameters $\bm t_n~(n = 1, \dots, N)$ randomly from an uniform distribution $U(\Delta^{M-1})$, and set up $\bm x_n$ as follows:
\begin{equation}\label{eqn:exp_add_noise}
    \bm x_n = \sum_{\bm d \in \N^M_D} \binom{D}{\bm d} \bm t_n^{\bm d} \bm p_{\bm d} + \bm \varepsilon_n \quad (n = 1, \dots, N),
\end{equation}
where $\bm \varepsilon_n~(n = 1, \dots, N)$ is a noise generated from a normal distribution $N(\bm 0, 0.1^2 \bm I)$.

For the inductive skeleton on the other hand, we set up training points for each dimensional subsimplex of $\Delta^{(m-1)} = \cup_{\card{\bm I} = m} \Delta^{\bm{I}}~(m = 1, \dots, M)$ to be learned.
First, we decoupled $N$ into $N^{(m)}~(m = 0, \dots, M-1)$.
To obtain parameters $\bm{t}^{(m)}_n~(n = 1, \dots, N^{(m)})$ from $\Delta^{(m)} = \cup_{\card{I} = m} \Delta^{\bm{I}}$, we further divided $N^{(m)}$ by the number of $m$-subsimplices $\Delta^{\bm{I}}~(\card{I} = m)$, and generated parameters of equal size from each uniform distribution $U(\Delta^{\bm{I}})$.
Then we produced training points $\bm x^{(m)}_n~(n = 1, \dots, N^{(m)})$ in the same way as \cref{eqn:exp_add_noise} and obtained training points $\set{(\bm x_n^{(m)}, \bm t_n^{(m)})}_{n=1}^{N^{(m)}}$ for each $(m)$-skeleton.

\subsubsection{Experimental settings}
Experiments were conducted on the following tuple $(N, M, L)$ with $D \in \Set{2, 3}$:
\begin{itemize}
    \item  $(M, L) = (8, 100)$ and $N \in \Set{250, 500, 1000, 2000}$;
    \item  $(N, L) = (1000, 100)$ and $M \in \Set{3, 4, 5, 6, 7, 8}$;
    \item  $(M, N) = (8, 1000)$ and $ M \in \Set{8, 25, 50, 100}$.
\end{itemize}

To verify the asymptotic risks derived in Section 3.1 and Section 3.2, we compared the following three methods:
\begin{description}
    \item[all-at-once] the all-at-once fitting (Section 3.1);
    \item[inductive skeleton (non-optimal)] the inductive skeleton fitting (Section 3.2) with $N^{(0)} = \dots = N^{(M-1)} = N / M$, which is not the minimizer of the risk of the inductive skeleton fitting;
    \item[inductive skeleton (optimal)] the inductive skeleton fitting (Section 3.2) where $N^{(0)}, N^{(1)}, \dots, N^{(M-1)}$ are determined by minimizing the risk of the inductive skeleton fitting under the constraints $\sum_{m = 0}^{M-1} N^{(m)} = N$ and $N^{(m)}\geq 0~(m = 0, \dots, M-1)$. The actual sample sizes $N^{(m)}$ for each $(M,D)$ are described in \cref{tab:optimal-subsample-ratio}.
    \Cref{tab:optimal-subsample-ratio} shows the optimal solutions of $N^{(m)}$ for each pairs $(M, D)$.
\end{description}
\begin{table}[ht]
    \centering
    \caption{Optimal subsample ratio for inductive skeleton fitting ($D$: degree of B\'ezier simplex, $M$: dimension of B\'ezier simplex, $N$: total sample size, $N^{(m)}$: sample size of $m$-skeleton).}
    \label{tab:optimal-subsample-ratio}
    \begin{tabular}{cccccccc}
        \toprule
        $D$     &$N^{(m)}$ &$M = 3$  &$M = 4$  &$M = 5$  &$M = 6$  &$M = 7$  &$M = 8$ \\ \midrule
        $2$     &$N^{(0)}$ &0.262$N$ &0.229$N$ &0.228$N$ &0.233$N$ &0.238$N$ &0.242$N$\\
                &$N^{(1)}$ &0.738$N$ &0.771$N$ &0.772$N$ &0.767$N$ &0.762$N$ &0.758$N$\\  \hline
        $3$     &$N^{(0)}$ &0.118$N$ &0.083$N$ &0.066$N$ &0.058$N$ &0.052$N$ &0.051$N$\\
                &$N^{(1)}$ &0.576$N$ &0.444$N$ &0.547$N$ &0.577$N$ &0.589$N$ &0.613$N$\\
                &$N^{(2)}$ &0.306$N$ &0.473$N$ &0.387$N$ &0.365$N$ &0.359$N$ &0.336$N$\\
        \bottomrule
    \end{tabular}
\end{table}

In this experiment, we estimated the B\'ezier simplex with degree $D = 2$ and 3 respectively.

\subsubsection{All results}\label{sec:results}
\Cref{fig:mse_vsN} shows box plots of MSEs over 20 trials and our theoretical risks for both the all-at-once fitting and the inductive skeleton fitting for each $N \in \Set{250, 500, 1000, 2000}$ with $(L, M) = (100, 8)$ and $D \in \Set{2, 3}$.
We observed that these figures empirically show that our theoretical risks are correct for both $D = 2$ and 3, and the gap between the actual MSEs and the risks are sufficiently small at $N = 1000$.
For both $D = 2$ and 3, the inductive skeleton (optimal) always achieved lower MSEs than that of the inductive skeleton (non-optimal).
This result suggests the efficiency of optimizing the risk of the inductive skeleton fitting with respect to the sample sizes $N^{(m)}$ of each dimension.
In addition, the inductive skeleton fitting (optimal) also outperformed the all-at-once fitting in the case of $D = 2$.
This result supports the discussion described in Section 3.3.
\begin{figure}[ht]
 \begin{minipage}{0.49\hsize}
        \centering
    \includegraphics[width=1\textwidth]{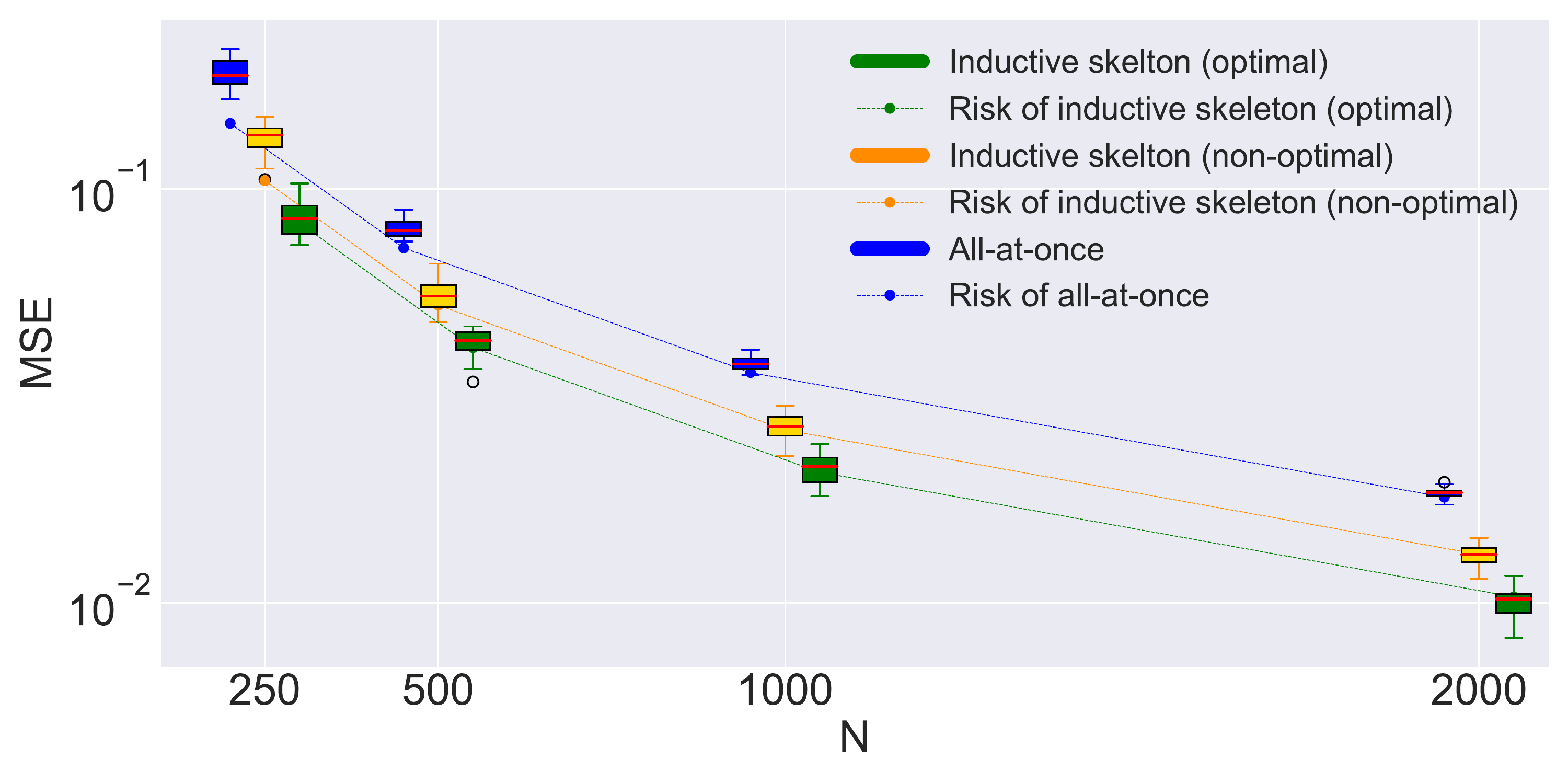}
    \subcaption{D=2}
    \label{fig:mse_vsN_d2}
 \end{minipage}
 \begin{minipage}{0.49\hsize}
        \centering
    \includegraphics[width=1\textwidth]{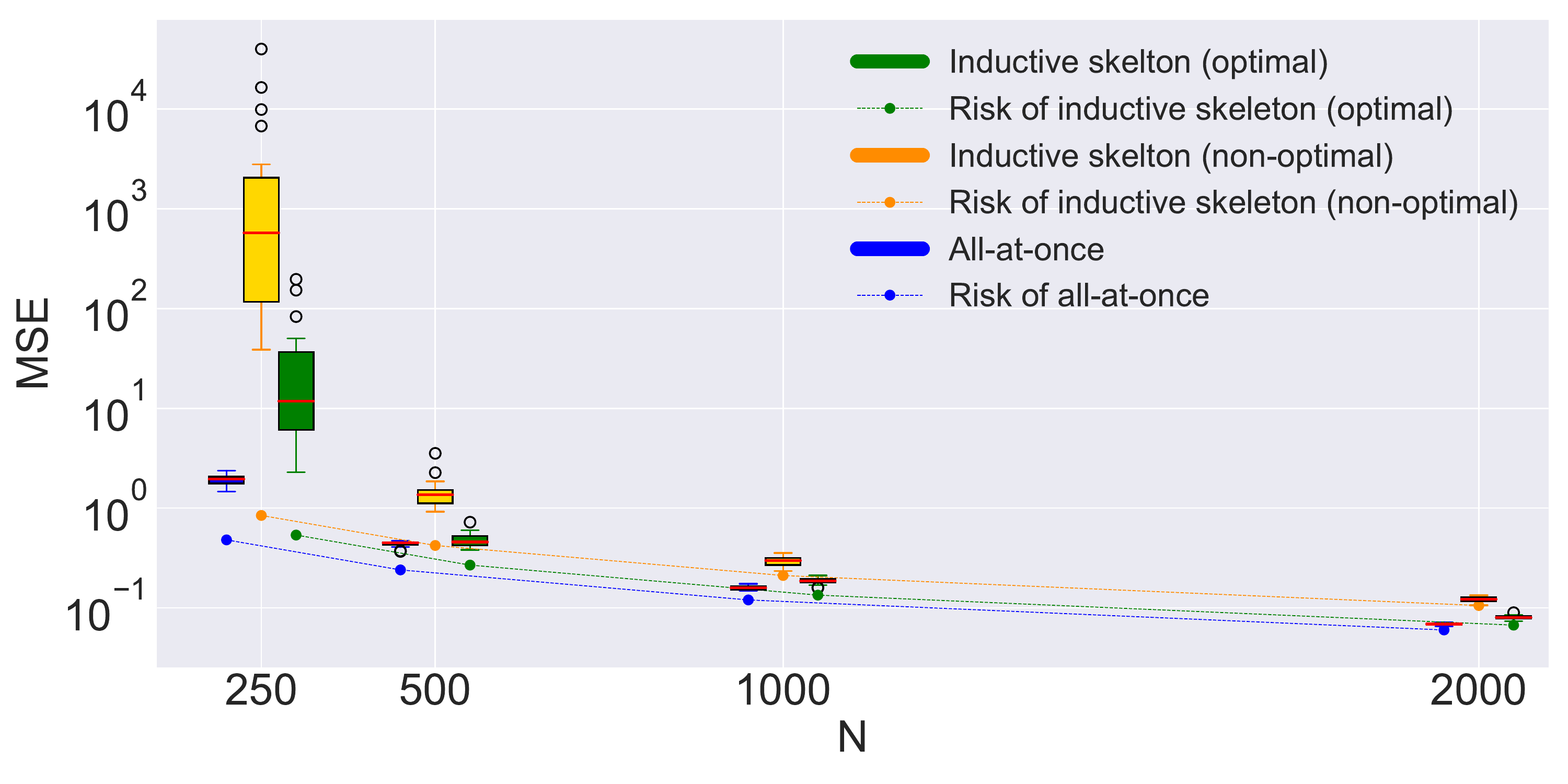}
    \subcaption{D=3}
    \label{fig:mse_vsN_d3}
 \end{minipage}
 \caption{Sample size $N$ vs. MSE with $(L,M)=(0.1,100,8)$ (boxplots over 20 trials and theoretical risks).}
 \label{fig:mse_vsN}
\end{figure}

\Cref{fig:mse_vsM} shows box plots of MSEs over 20 trials and our theoretical risks for each $M \in \Set{3, 4, 5, 6, 7, 8}$ with $(L, N) = (100, 1000)$.
As well as \cref{fig:mse_vsN}, the inductive skeleton fitting always outperformed the all-at-once fitting in the case of $D = 2$.
Furthermore, the difference of MSEs between the inductive skeleton fitting and the all-at-once fitting gets wider as $M$ grows.
This suggests that the inductive skeleton fitting with $D = 2$ approximates more effectively than the all-at-once fitting does for a high-dimensional $M$.
\begin{figure}[ht]
 \begin{minipage}{0.49\hsize}
        \centering
    \includegraphics[width=1\textwidth]{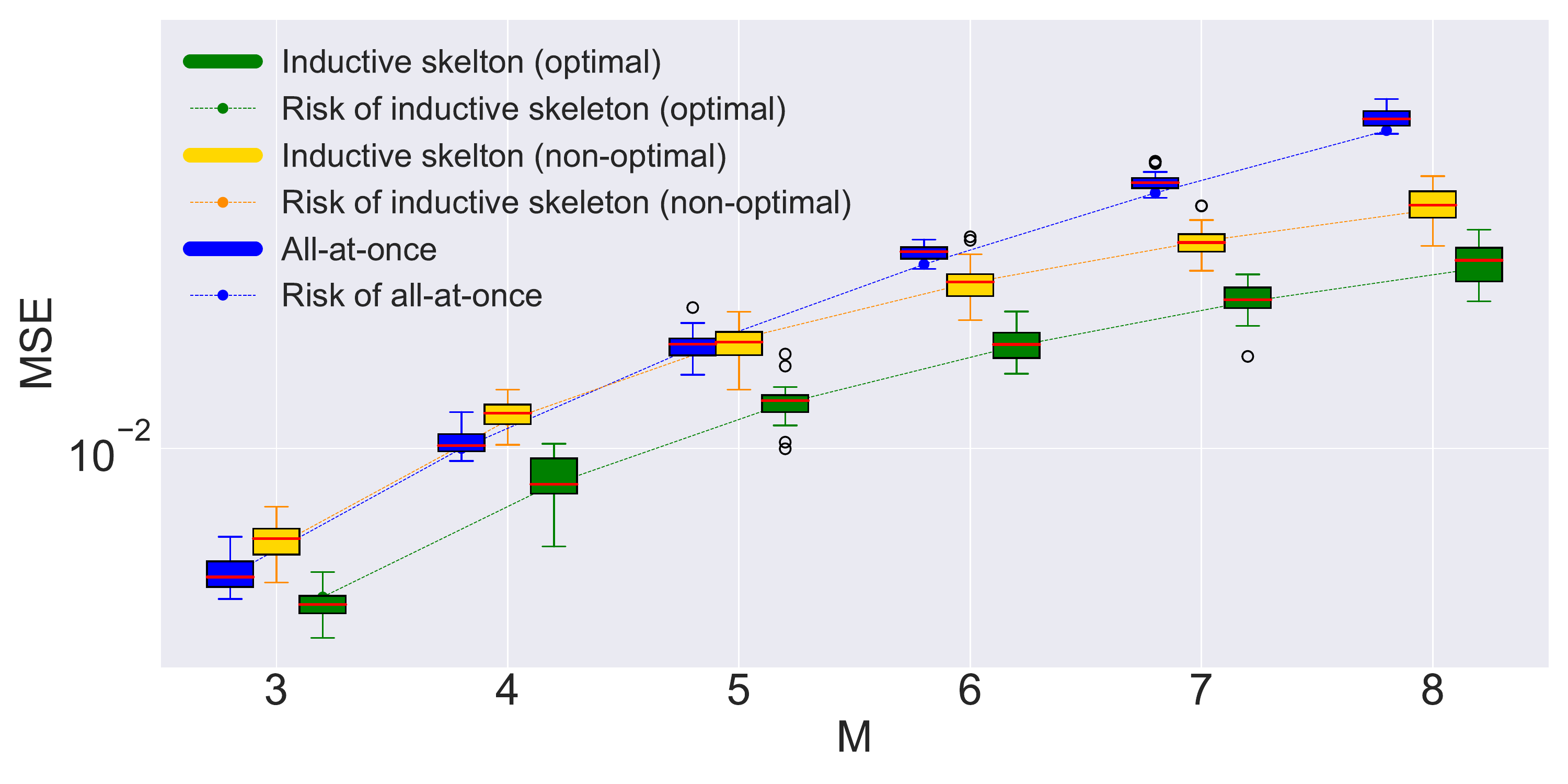}
    \subcaption{$D = 2$}    \label{fig:mse_vsM_d2}
 \end{minipage}
 \begin{minipage}{0.49\hsize}
        \centering
    \includegraphics[width=1\textwidth]{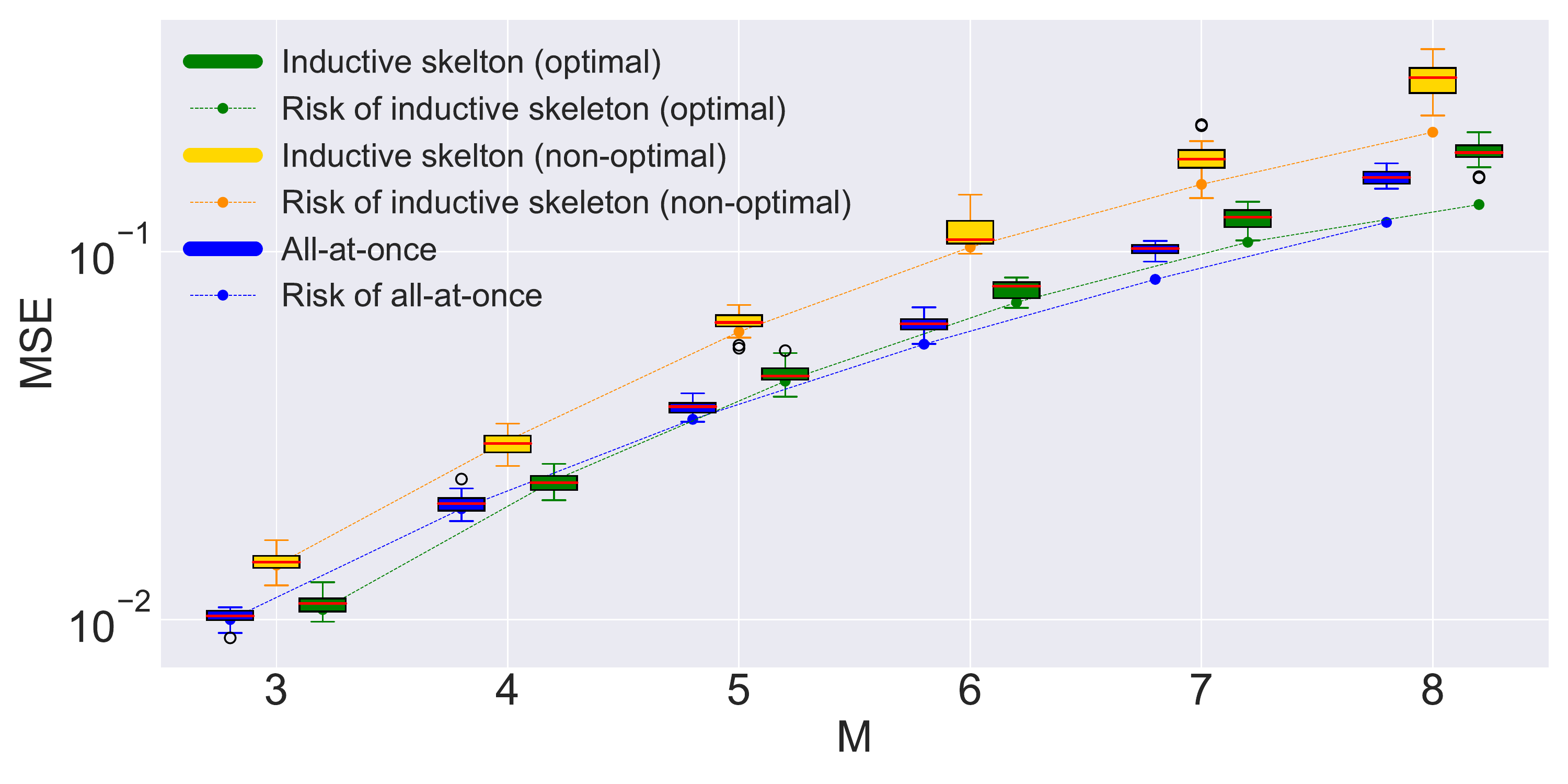}
    \subcaption{$D = 3$}    \label{fig:mse_vsM_d3}
 \end{minipage}
 \caption{Dimension of a simplex $M$ vs. MSE with $(L, N) = (100, 1000)$ (boxplots over 20 trials and theoretical risks).}
 \label{fig:mse_vsM}
\end{figure}

\Cref{fig:mse_vsL} shows box plots of MSEs over 20 trials and our theoretical risks for each $L\in \Set{8, 25, 50, 100}$ and $D = 2, 3$ with $(M, N)=(8, 1000)$.
As well as \cref{fig:mse_vsN,fig:mse_vsM}, the inductive skeleton fitting always outperform the all-at-once fitting in the case of $D = 2$.

\begin{figure}[ht]
 \begin{minipage}{0.49\hsize}
        \centering
    \includegraphics[width=1\textwidth]{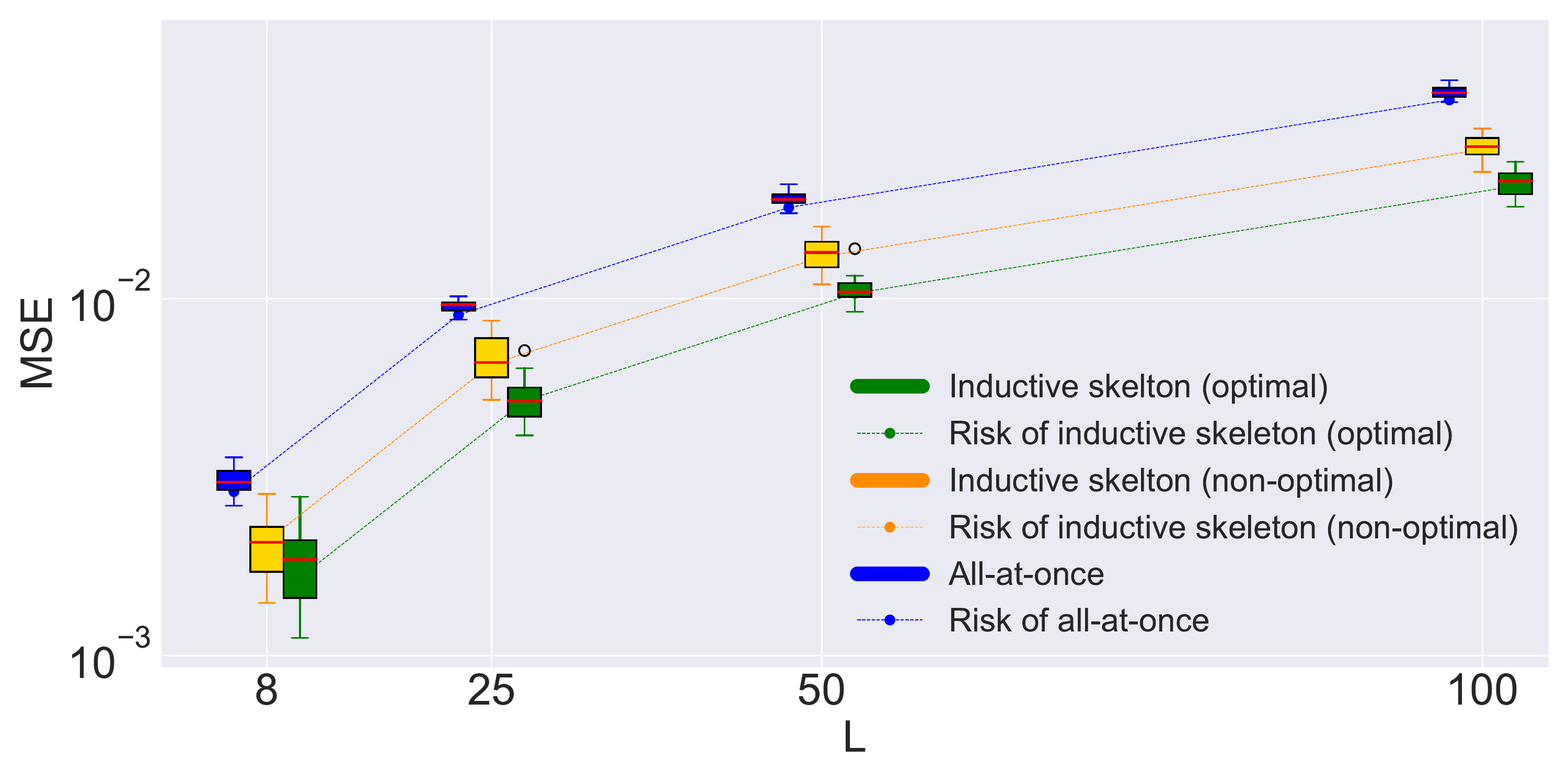}
    \subcaption{$D = 2$}
    \label{fig:mse_vsL_d2}
 \end{minipage}
 \begin{minipage}{0.49\hsize}
        \centering
    \includegraphics[width=1\textwidth]{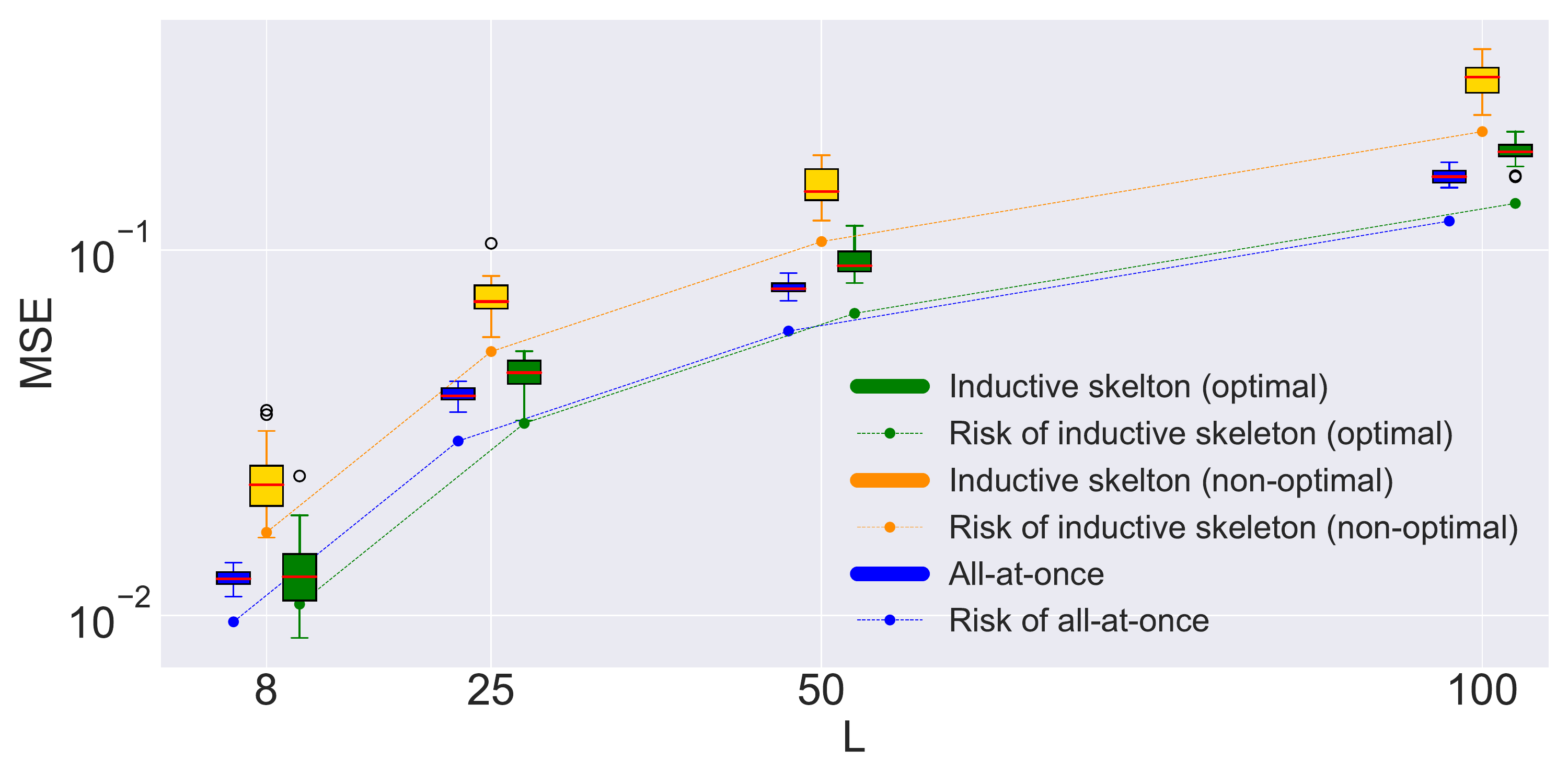}
    \subcaption{$D = 3$}
    \label{fig:mse_vsL_d3}
 \end{minipage}
 \caption{Dimension of control points $L$ vs. MSE with sample size $(M, N) = (3, 8, 1000)$ (boxplots over 20 trials and theoretical risks).}
\label{fig:mse_vsL}
\end{figure}

\subsection{Multi-objective optimization instances}\label{appd:sec:MOP-instances}
Next, we describe the experiment protocol for the multi-objective optimization instances used in \cref{sec:MOP-instances} of the main paper.
We only present the data generation process since the fitting and evaluation part are the same as the synthetic cases.

\subsubsection{A generalized location problem}\label{sec:location-problem}
We generalized the multi-objective location problem \cite{Kuhn1967} to a higher dimension:
\begin{equation}
\begin{split}
\text{minimize } & f(x) = (f_1(x), f_2(x), f_3(x)) \text{ subject to }x \in \R^4\\
\text{where }    & f_m(x) = \norm{x - e_m}^2 \quad (m = 1, \dots, 3)\\
                 & e_1 = (1,0,0,0),\ e_2 = (0,1,0,0),\ e_3 = (0,0,1,0).
\end{split}
\end{equation}
Note that this is a special case of the MED benchmark problem \cite{Hamada2010}.
The MED problem is simplicial \cite{Hamada2017} and its Pareto set is known to be the convex hull of the minimizers of separate objective functions, i.e., the 2-simplex spanned by $e_1, e_2, e_3$.
For each vertex, edge, face of this simplex, which is the Pareto set of each 1-, 2-, 3-objective subproblem, we generate a subsample according to the uniform distribution on it.

\subsubsection{The group lasso}\label{sec:group-lasso}
We applied the B\'ezier simplex fittings to estimation of the hyper-parameter space of a sparse regression method.
The dataset used in this experiment was \texttt{Birthwt} in the R-package \texttt{MASS}, which contains 189 births at the Baystate Medical Centre, Springfield, Massachusetts during 1986 \cite{Hosmer1989,Venables2002}.
From the dataset, we adopted six continuous features \texttt{age1}, \texttt{age2}, \texttt{age3}, \texttt{lwt1}, \texttt{lwt2}, \texttt{lwt3} as predictors and one continuous feature \texttt{bwt} as a response for regression analysis.
Since the predictors are classified into two groups, \texttt{age} and \texttt{lwt}, the group lasso~\cite{Yuan2006} was employed.

Put $N=189$ and $M=6$.
Let $A$ be an $N \times M$ matrix of observations of the predictors, $x \in \R^M$ be a row vector of the predictor coefficients to be estimated, separated into two groups $x_\text{age} = (x_1, x_2, x_3)^\top$ and $x_\text{lwt} = (x_4, x_5, x_6)^\top$, and $y \in \R^N$ be a row vector of observations of the response.
The group lasso regressor is the solution to the following problem:
\begin{equation}\label{eqn:group-lasso}
\text{minimize } \frac{1}{2N} \norm{Ax - y}^2 + \frac{\lambda}{\sqrt{3}} \paren{\norm{x_\text{age}} + \norm{x_\text{lwt}}} \text{ subject to } x \in \R^6
\end{equation}
where $\norm{\cdot}$ is the Euclidean norm, and $\lambda$ is a positive number to be tuned by users.
This original form suffers from two drawbacks:
\begin{itemize}
    \item Choosing an appropriate value for $\lambda$ involves a grid search on an unbounded domain.
    \item Since two groups have physically different units of measurement, same weights are not always appropriate even if their values are normalized.
\end{itemize}

Instead, we consider each term in \cref{eqn:group-lasso} as a separate objective function:
\begin{equation}\label{eqn:group-lasso-mop}
\begin{split}
    \text{minimize } & f(x)   = (f_1(x), f_2(x), f_3(x)) \text{ subject to } x \in \R^6\\
    \text{where }    & f_1(x) = \norm{Ax - y}^2,\ f_2(x) = \norm{x_\text{age}}^2,\ f_3(x) = \norm{x_\text{lwt}}^2.
\end{split}
\end{equation}
Notice that the use of the squared norm in $f_2$ and $f_3$ does not change their solutions.
It is easy to see that every objective function in \cref{eqn:group-lasso-mop} is convex but not strongly convex.
We make them strongly convex by the following perturbation:
\begin{align*}
    \tilde f_1 &= f_1 + \varepsilon \norm{x}^2,\\
    \tilde f_2 &= f_2 + \varepsilon \norm{x}^2,\\
    \tilde f_3 &= f_3 + \varepsilon \norm{x}^2
\end{align*}
where $\varepsilon$ is an arbitrarily small positive number (we set $\varepsilon=10^{-4}$).
Now the problem minimizing a mapping $\tilde f= (\tilde f_1, \tilde f_2, \tilde f_3)$ is strongly convex.
By \cite[Theorems 1.1 and 3.1]{Hamada2019}, this problem is weakly simplicial and the mapping
\begin{equation}\label{eqn:group-lasso-sop}
    x^*(w) = \arg\min_x \inprod{w}{f(x)}
\end{equation}
is well-defined and continuous on $\Delta^2$, satisfying $x^*(\Delta^2_I) = X^*(\tilde f_I)$ for all $I \subseteq \set{1,2,3}$.

Then, we obtained subsamples by solving \cref{eqn:group-lasso-sop} repeatedly with varying $w \in \Delta^2_I$ for each $I \subseteq \set{1, 2, 3}$.
For each such $I$, the weight $w$ was drawn from the uniform distribution on $\Delta^2_I$ and the problem \cref{eqn:group-lasso-sop} was solved by the steepest descent method.

The same idea can be applied to a broad range of sparse learning methods, including the original lasso \cite{Tibshirani1996}, the fused lasso \cite{Tibshirani2005}, the smooth lasso \cite{Hebiri2011}, and the elastic net \cite{Zou2005}.
For those methods, their group-wise regularization terms can be considered as separate objectives, and the resulting problems would be many-objective (four-objective or more) where the all-at-once fitting will much outperform over the inductive skeleton fitting.
We however remark that the bridge regression \cite{Frank1993} is not the case since its regularization term using a nonconvex $\ell_p$-norm (i.e., $p < 1$) cannot change into a strongly convex function via perturbations.

\subsubsection{The Pareto fronts of each multi-objective optimization problem}\label{sec:Pareto-fronts}
\Cref{fig:Pareto-fronts} shows the Pareto fronts of the location problem and the group lasso.
From \cref{fig:Pareto-fronts}, we can see that the Pareto front of the location problem can be represented by a B\'ezier simplex of degree $D=2$.
For the group lasso on the other hand, its Pareto front cannot be represented by a a B\'ezier simplex of degree $D=2$ but of $D=3$.
\begin{figure}[ht]
 \begin{minipage}{0.49\hsize}
        \centering
    \includegraphics[width=1\textwidth]{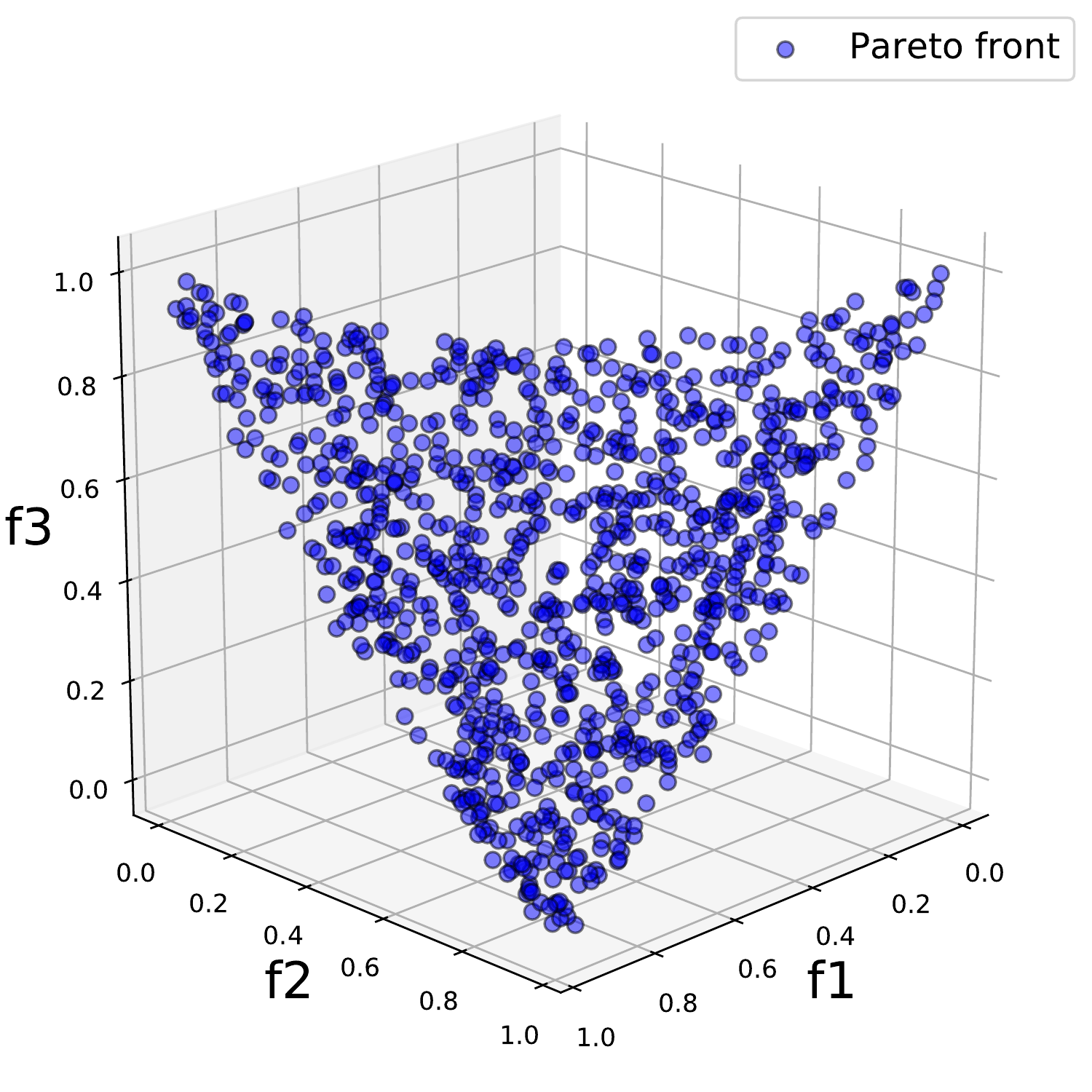}
    \subcaption{Location problem}
    \label{fig:Pareto-front-location-problem}
 \end{minipage}
 \begin{minipage}{0.49\hsize}
        \centering
    \includegraphics[width=1\textwidth]{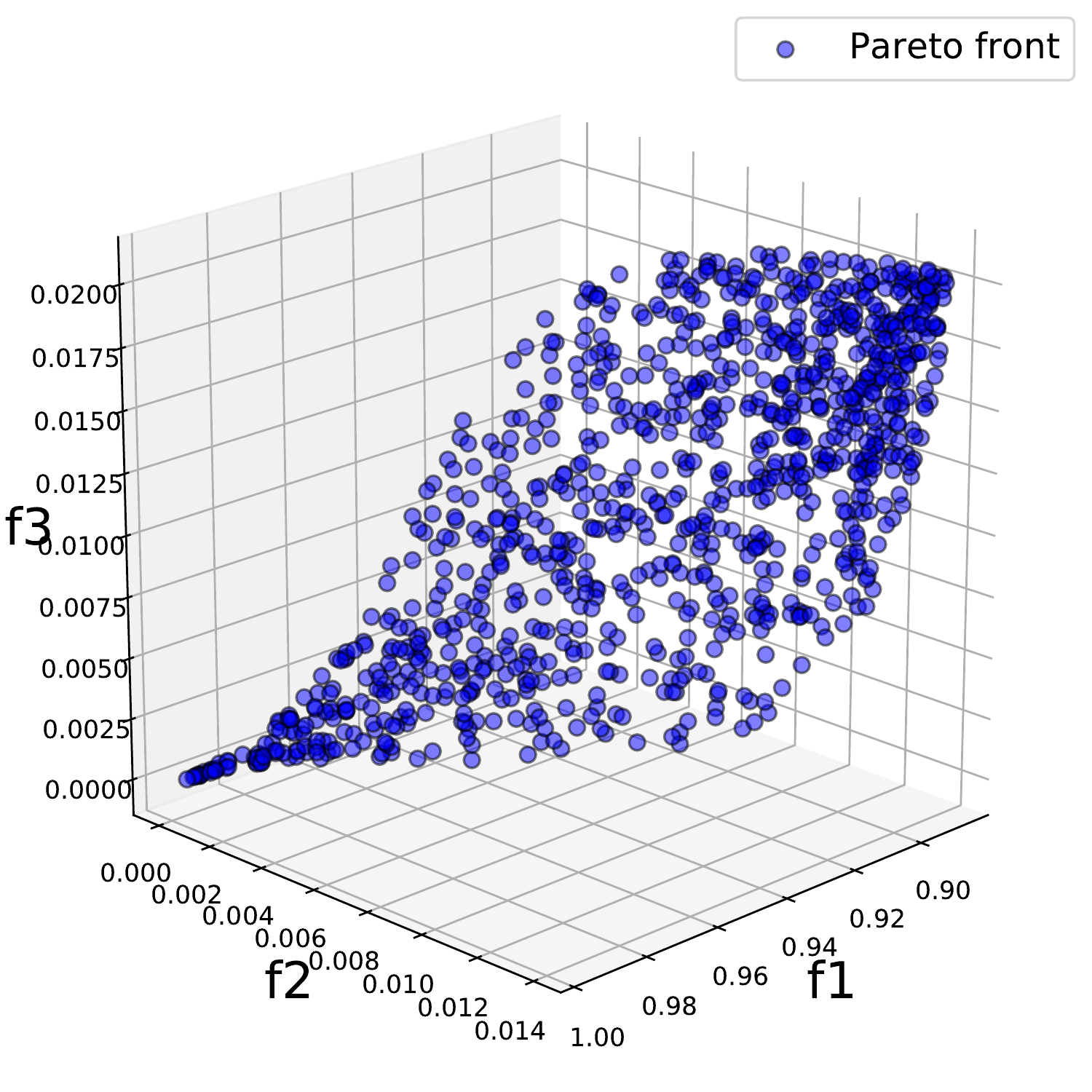}
    \subcaption{Group lasso}
    \label{fig:Pareto-front-group-lasso}
 \end{minipage}
 \caption{The Pareto fronts of the location problem and the group lasso.}
\label{fig:Pareto-fronts}
\end{figure}


\begin{thebibliography}{100}

\bibitem{Borges2002}
Carlos~F. Borges and Tim Pastva.
\newblock Total least squares fitting of {B\'ezier} and {B-spline} curves to
  ordered data.
\newblock {\em Computer Aided Geometric Design}, 19(4):275--289, 2002.

\bibitem{Deb2001}
K.~Deb.
\newblock {\em Multi-objective Optimization Using Evolutionary Algorithms}.
\newblock John Wiley \& Sons, Inc., New York, NY, USA, 2001.

\bibitem{Deb2014}
K.~Deb and H.~Jain.
\newblock An evolutionary many-objective optimization algorithm using
  reference-point-based nondominated sorting approach, part {I}: Solving
  problems with box constraints.
\newblock {\em IEEE Transactions on Evolutionary Computation}, 18(4):577--601,
  2014.

\bibitem{Eichfelder2008}
G.~Eichfelder.
\newblock {\em Adaptive Scalarization Methods in Multiobjective Optimization}.
\newblock Springer-Verlag, Berlin, Heidelberg, 2008.

\bibitem{Hamada2017}
Naoki Hamada.
\newblock Simple problems: The simplicial gluing structure of {Pareto} sets and
  {Pareto} fronts.
\newblock In {\em Proceedings of the Genetic and Evolutionary Computation
  Conference Companion}, GECCO '17, pages 315--316, New York, NY, USA, 2017.
  ACM.

\bibitem{Hamada2019}
Naoki Hamada, Kenta Hayano, Shunsuke Ichiki, Yutaro Kabata, and Hiroshi
  Teramoto.
\newblock Topology of {Pareto} sets of strongly convex problems.
\newblock {\em ArXiv e-prints}, 2019.
\newblock \url{http://arxiv.org/abs/1904.03615}.

\bibitem{Hamada2010}
Naoki Hamada, Yuichi Nagata, Shigenobu Kobayashi, and Isao Ono.
\newblock Adaptive weighted aggregation: A multiobjective function optimization
  framework taking account of spread and evenness of approximate solutions.
\newblock In {\em Proceedings of the 2010 IEEE Congress on Evolutionary
  Computation}, CEC 2010, pages 787--794, 2010.

\bibitem{Harada2007}
K.~Harada, J.~Sakuma, S.~Kobayashi, and I.~Ono.
\newblock Uniform sampling of local {Pareto}-optimal solution curves by
  {Pareto} path following and its applications in multi-objective {GA}.
\newblock In {\em Proceedings of the Genetic and Evolutionary Computation
  Conference (GECCO)}, pages 813--820, New York, NY, USA, 2007. ACM.

\bibitem{Harada2006}
Ken Harada, Jun Sakuma, and Shigenobu Kobayashi.
\newblock Local search for multiobjective function optimization: {Pareto}
  descent method.
\newblock In {\em Proceedings of the 8th Annual Conference on Genetic and
  Evolutionary Computation}, GECCO '06, pages 659--666, New York, NY, USA,
  2006. ACM.

\bibitem{Hebiri2011}
Mohamed Hebiri and Sara van~de Geer.
\newblock The smooth-lasso and other $\ell_1 + \ell_2$-penalized methods.
\newblock {\em Electron. J. Statist.}, 5:1184--1226, 2011.

\bibitem{Hernandez-Lobato2016}
Daniel Hernandez-Lobato, Jose Hernandez-Lobato, Amar Shah, and Ryan Adams.
\newblock Predictive entropy search for multi-objective bayesian optimization.
\newblock In Maria~Florina Balcan and Kilian~Q. Weinberger, editors, {\em
  Proceedings of The 33rd International Conference on Machine Learning},
  volume~48 of {\em Proceedings of Machine Learning Research}, pages
  1492--1501, New York, New York, USA, 2016. PMLR.

\bibitem{Hillermeier2001}
C.~Hillermeier.
\newblock {\em Nonlinear Multiobjective Optimization: A Generalized Homotopy
  Approach}, volume~25 of {\em International Series of Numerical Mathematics}.
\newblock Birkh\"{a}user Verlag, Basel, Boston, Berlin, 2001.

\bibitem{Hosmer1989}
David~W Hosmer and Stanley Lemeshow.
\newblock {\em Applied Logistic Regression}.
\newblock Wiley, New York, 1989.

\bibitem{Kobayashi2019}
Ken Kobayashi, Naoki Hamada, Akiyoshi Sannai, Akinori Tanaka, Kenichi Bannai,
  and Masashi Sugiyama.
\newblock B\'ezier simplex fitting: Describing {Pareto} fronts of simplicial
  problems with small samples in multi-objective optimization.
\newblock In {\em Proceedings of the Thirty-Third {AAAI} Conference on
  Artificial Intelligence}, AAAI-19, to appear.

\bibitem{Kuhn1967}
Harold~W Kuhn.
\newblock On a pair of dual nonlinear programs.
\newblock {\em Nonlinear Programming}, 1:38--45, 1967.

\bibitem{Frank1993}
lldiko E.~Frank and Jerome~H. Friedman.
\newblock A statistical view of some chemometrics regression tools.
\newblock {\em Technometrics}, 35(2):109--135, 1993.

\bibitem{Mastroddi2013}
F.~Mastroddi and S.~Gemma.
\newblock Analysis of {Pareto} frontiers for multidisciplinary design
  optimization of aircraft.
\newblock {\em Aerosp. Sci. Technol.}, 28(1):40--55, 2013.

\bibitem{Miettinen1999}
Kaisa~M. Miettinen.
\newblock {\em Nonlinear Multiobjective Optimization}, volume~12 of {\em
  International Series in Operations Research \& Management Science}.
\newblock Springer-Verlag, GmbH, 1999.

\bibitem{Shoval2012}
O.~Shoval, H.~Sheftel, G.~Shinar, Y.~Hart, O.~Ramote, A.~Mayo, E.~Dekel,
  K.~Kavanagh, and U.~Alon.
\newblock Evolutionary trade-offs, {Pareto} optimality, and the geometry of
  phenotype space.
\newblock {\em Science}, 336(6085):1157--1160, 2012.

\bibitem{Tibshirani1996}
Robert Tibshirani.
\newblock Regression shrinkage and selection via the lasso.
\newblock {\em Journal of the Royal Statistical Society. Series B
  (Methodological)}, 58(1):267--288, 1996.

\bibitem{Tibshirani2005}
Robert Tibshirani, Michael Saunders, Saharon Rosset, Ji~Zhu, and Keith Knight.
\newblock Sparsity and smoothness via the fused lasso.
\newblock {\em Journal of the Royal Statistical Society: Series B (Statistical
  Methodology)}, 67(1):91--108, 2005.

\bibitem{Venables2002}
W.~N. Venables and B.~D. Ripley.
\newblock {\em Modern Applied Statistics with {S}}.
\newblock Springer, fourth edition, 2002.

\bibitem{Vrugt2003}
Jasper~A. Vrugt, Hoshin~V. Gupta, Luis~A. Bastidas, Willem Bouten, and Soroosh
  Sorooshian.
\newblock Effective and efficient algorithm for multiobjective optimization of
  hydrologic models.
\newblock {\em Water Resources Research}, 39(8):1214--1232, 2003.

\bibitem{Yang2019}
Kaifeng Yang, Michael Emmerich, Andr\'e Deutz, and Thomas B\"ack.
\newblock Multi-objective {Bayesian} global optimization using expected
  hypervolume improvement gradient.
\newblock {\em Swarm and Evolutionary Computation}, 44:945--956, 2019.

\bibitem{Yuan2006}
Ming Yuan and Yi~Lin.
\newblock Model selection and estimation in regression with grouped variables.
\newblock {\em Journal of the Royal Statistical Society: Series B (Statistical
  Methodology)}, 68(1):49--67, 2006.

\bibitem{Zhang2007}
Q.~Zhang and H.~Li.
\newblock {MOEA/D}: A multiobjective evolutionary algorithm based on
  decomposition.
\newblock {\em IEEE Transactions on Evolutionary Computation}, 11(6):712--731,
  2007.

\bibitem{Zou2005}
Hui Zou and Trevor Hastie.
\newblock Regularization and variable selection via the elastic net.
\newblock {\em Journal of the Royal Statistical Society. Series B (Statistical
  Methodology)}, 67(2):301--320, 2005.

\end{thebibliography}
\end{document}